\documentclass[conference]{IEEEtran}
\IEEEoverridecommandlockouts
\usepackage{cite}
\usepackage{amsmath,amssymb,amsfonts}
\usepackage{amsthm}
\newtheorem{proposition}{Proposition}
\newtheorem{corollary}{Corollary}
\newtheorem{remark}{Remark}
\usepackage{algorithmic}
\usepackage{graphicx}
\usepackage{textcomp}
\usepackage{xcolor}
\usepackage{tikz}
\usetikzlibrary{arrows.meta,calc}
\usepackage{booktabs}
\usepackage[caption=false]{subfig}
\usepackage{url}
\usepackage{enumitem}
\usepackage[colorlinks=false,hidelinks,bookmarks=true]{hyperref}
\hypersetup{
  pdftitle={Quantifying Motion Excitation for Metric Scale Observability
in Monocular Visual-Inertial Odometry},
  pdfauthor={Hadush Hailu, Bruk Gebregziabher, Siddhartha Gudipudi, Manoj Bhatta}
}
\def\BibTeX{{\rm B\kern-.05em{\sc i\kern-.025em b}\kern-.08em
    T\kern-.1667em\lower.7ex\hbox{E}\kern-.125emX}}
\begin{document}

\title{Quantifying Motion Excitation for Metric Scale Observability
in Monocular Visual-Inertial Odometry\\
}

\author{
\IEEEauthorblockN{
Hadush Hailu$^{1}$,
Bruk Gebregziabher$^{2}$,
Siddhartha Gudipudi$^{3}$,
Manoj Bhatta$^{4}$
}

\vspace{1em}

\IEEEauthorblockA{
\begin{tabular}{l}
$^{1}$ Computer Science, Maharishi International University, Iowa, USA, hadush.gebrerufael@miu.edu \\[1mm]
$^{2}$ Electrical Engineering and Computing, University of Zagreb, Zagreb, Croatia, bruk.gebregziabher@fer.hr \\[1mm]
$^{3}$ Computer Science, Iowa State University, Iowa, USA, gs2196@outlook.com \\[1mm]
$^{4}$ Computer Science and Application, Utkal University, India, mkbhatta@gmail.com
\end{tabular}
}
}
\maketitle

\begin{abstract}
Monocular visual--inertial odometry (VIO) cannot recover metric scale from
vision alone; scale must be resolved through inertial measurements.
We present a trajectory-dependent observability analysis showing that
translational acceleration, produced by curvature, not constant-speed
straight-line travel, is the fundamental source that couples scale to the
inertial state.
This relationship is formalized through the gravity--acceleration asymmetry in the IMU model,
from which we derive rank conditions on the observability matrix and propose a lightweight
excitation metric computable from raw IMU data.
Controlled experiments on a differential-drive robot with a monocular camera
and consumer-grade IMU validate the theory, with straight-line motion yielding
9.2\% scale error, circular motion 6.4\%, and figure-eight motion 4.8\%,
with excitation spanning four orders of magnitude.
These results establish trajectory design as a practical mechanism for improving metric scale
recovery.
\end{abstract}


\section{Introduction}
\label{sec:introduction}

Monocular visual--inertial odometry (VIO) fuses a single camera with an
inertial measurement unit (IMU) to estimate six-degree-of-freedom motion
without external positioning systems, making it a cornerstone of lightweight
robot navigation~\cite{Campos2021ORBSLAM3,Rosinol2021KimeraVIO,Geneva2022OpenVINS}.
Because a single camera provides only bearing measurements, the reconstructed
three-dimensional structure is defined only up to an unknown scale
factor~\cite{Martinelli2014Observability,Li2013Observability}.
Tightly coupled IMU fusion resolves this ambiguity in principle, since
accelerometer readings couple metric acceleration with gravitational
reference~\cite{Mourikis2007MSCKF,Forster2021PreintegrationSurvey}.
In practice, however, whether the scale factor is accurately recovered depends
on the richness of the inertial signals, which is itself a function of the
executed trajectory~\cite{Jones2011VIO}.

Existing observability analyses have established that monocular VIO possesses
four unobservable directions (global position and heading) under generic
motion and that constant-velocity travel introduces an additional
degeneracy in which metric scale becomes
unidentifiable~\cite{Martinelli2014Observability,Jones2011VIO,Li2013Observability}.
Consistency-preserving estimator designs respect these structural
properties~\cite{Hesch2014Consistency}, and recent work has begun to
characterize degenerate configurations more
broadly~\cite{GroundVIO2024TITS,Delaune2021RangeVIO}.
Nevertheless, these studies focus primarily on which states are
observable rather than on how much information a given trajectory
provides about scale.
As a result, practitioners lack quantitative guidance for designing motions
that maximize scale conditioning.
Field reports confirm that the gap is practically significant, with scale
estimation failing silently on ground robots whose near-planar,
piecewise-constant-velocity motion produces weak inertial
excitation~\cite{GroundVIO2024TITS,Delaune2021RangeVIO}.

Common remedies add hardware such as wheel encoders and range sensors, or learned
depth priors~\cite{Jiang2021PanoramicVIOWheel,Scheiber2021MidAirRangeInit,
Merrill2023RSSInit,Hao2023RangeVIOCoarseFine}.
These approaches treat scale as a sensor problem.
This paper advances a complementary view, arguing that scale is a motion problem.

Our central thesis is that motion is a sensing modality for metric
scale.
The standard IMU measurement
model~\cite{Forster2021PreintegrationSurvey,Barfoot2021StateEstimation}
contains a gravity--acceleration asymmetry in which translational
acceleration is proportional to the unknown scale factor~$s$, whereas
gravity provides a fixed, scale-independent reference.
This asymmetry has been noted in observability
proofs~\cite{Martinelli2014Observability,Kelly2011VINSurvey},
yet its implications for trajectory design have not been quantified.
We distill the dependence into a scalar excitation index~$E$
computable from raw IMU data and show, both analytically and
experimentally, that trajectories with time-varying curvature maximize~$E$
and minimize scale error.

Controlled experiments on a differential-drive robot equipped with a
monocular camera and a consumer-grade BNO055 IMU validate the
analysis (Section~\ref{sec:experiments}). Across matched 3\,m
trajectories, figure-eight motion reduces scale error by 48\% relative
to straight-line travel (Table~\ref{tab:scale},
Fig.~\ref{fig:excitation_vs_error}), with excitation spanning four
orders of magnitude (Table~\ref{tab:excitation}).
These results are obtained without wheel odometry, depth sensors,
or learned priors, isolating the contribution of trajectory design alone.

This paper makes three contributions.
First, building on the observability framework
of~\cite{Martinelli2014Observability,Jones2011VIO}, we provide a formal
analysis of the gravity--acceleration asymmetry and show that the Fisher
information contributed by each IMU sample toward the scale parameter is
proportional to $\|\dot{\mathbf{v}}_W\|^2$
(Proposition~\ref{prop:scale_info}).
Second, we derive a trajectory-dependent excitation index~$E$ that
quantifies scale conditioning from raw IMU signals, moving beyond the
binary observable/unobservable classification of prior
work~\cite{Li2013Observability}.
Third, controlled experiments on a minimally instrumented ground robot
(Fig.~\ref{fig:robot_platform})
confirm a monotonic relationship between excitation and scale accuracy
across three canonical motion patterns
(Figs.~\ref{fig:scale_bar}--\ref{fig:scale_convergence}), establishing
that motion planning which maximizes time-varying curvature can
substitute for additional sensing hardware in recovering metric scale
(Corollary~\ref{cor:excitation}).

\section{Related Work}
\label{sec:related_work}

\subsection{Geometry-Based Visual--Inertial SLAM and VIO}
Geometry-based visual--inertial SLAM systems form the backbone of modern state
estimation pipelines.
ORB-SLAM3 provides a widely used open-source framework supporting visual,
visual--inertial, and multi-map SLAM
\cite{Campos2021ORBSLAM3}.
Fixed-lag smoothing approaches such as Kimera-VIO achieve high accuracy while
maintaining real-time performance and extensibility
\cite{Rosinol2021KimeraVIO}.
Filtering-based methods, including MSCKF-style estimators and recent pose-only
representations, remain attractive for embedded platforms due to their computational
efficiency
\cite{Liu2022MSCKFPreint}.

\subsection{Scale Observability, Initialization, and Conditioning}
Formal observability analyses have established that monocular visual--inertial
systems possess four unobservable directions (global position and yaw) under
generic motion, but that metric scale becomes unobservable when translational
acceleration vanishes \cite{Martinelli2014Observability,Jones2011VIO}.
Consistency-preserving estimators have been designed to respect these
observability properties \cite{Hesch2014Consistency}.
Range-assisted methods explicitly address scale by introducing direct metric
constraints, enabling scale recovery even under constant-velocity motion
\cite{Delaune2021RangeVIO,Hao2023RangeVIOCoarseFine}.
Initialization remains a major failure mode, motivating approaches that leverage
learned single-view depth to stabilize early estimation
\cite{Merrill2023RSSInit}.
Observability analyses have also been applied to understand which states can be
reliably estimated under realistic motion and sensing conditions
\cite{DefVINS2026}.

\subsection{Motion Excitation and Trajectory Effects}
The dependence of estimation quality on motion richness has been explored across
calibration and VIO literature.
Ground-VIO explicitly studies monocular visual--inertial odometry under planar
constraints and highlights the limitations imposed by weak excitation
\cite{GroundVIO2024TITS}.
However, most prior work treats motion as incidental rather than as a design variable.
Our work differs by treating motion as an information source and by linking
trajectory curvature and time variation directly to scale observability.

\subsection{Camera, IMU, and Spatiotemporal Calibration}
Accurate calibration is essential for reliable visual--inertial estimation.
Wide-angle and fisheye camera models have been extensively studied, with equidistant
models shown to be well suited for large field-of-view optics
\cite{Huai2024WideAngleReview,FisheyeOverview2023}.
Online and offline calibration of camera--IMU extrinsics and temporal offsets remains
an active area of research
\cite{Yang2022FullCalib,Yang2024MultiVINS,eKalibrInertial2025}.
IMU noise characterization via Allan variance continues to be the standard approach
for extracting noise density and bias instability parameters
\cite{Suvorkin2024Allan}.

\subsection{Learning-Based and Hybrid Approaches}
Learning-based SLAM and VIO methods have demonstrated impressive robustness in
challenging environments.
DROID-SLAM achieves strong performance through dense learned correspondence matching
\cite{Teed2021DROIDSLAM}.
Lightweight learning-based VIO methods aim to reduce computational overhead while
maintaining accuracy
\cite{Park2024ULVIO}.
Hybrid approaches incorporating learned bias prediction or learned features further
improve robustness
\cite{DeepIMUBias2022,Altawaitan2025BiasPrediction,Li2025DeepLineVIO}.
Event-based and motion-aware extensions continue to expand the scope of visual--
inertial estimation
\cite{Shahraki2025MAEVIO}.
While these methods inject powerful priors, they do not eliminate the fundamental
geometric scale ambiguity of monocular vision.

In summary, existing work treats scale recovery primarily as a sensor or
algorithm design problem.
The present paper offers a complementary, dynamics-first perspective.
By analyzing how trajectory geometry modulates the information content of
inertial measurements, we provide actionable guidance for motion design in
minimally instrumented monocular visual--inertial systems.

\section{Theoretical Analysis of Metric Scale Observability}
\label{sec:theory}

This section analyzes how metric scale becomes observable in monocular
visual--inertial systems and how this observability depends on the executed
motion. We begin by defining the inertial navigation model, then analyze the
limitations of monocular vision, and finally show how inertial dynamics inject
scale information through motion-dependent excitation.

\subsection{State Definition and Inertial Dynamics}

We consider a rigid body equipped with a monocular camera and an IMU.
The inertial navigation state is defined as
\begin{equation}
\mathbf{x}(t) =
\left\{
\mathbf{R}_{WI}(t),
\mathbf{v}_W(t),
\mathbf{p}_W(t),
\mathbf{b}_g(t),
\mathbf{b}_a(t)
\right\},
\label{eq:state_def}
\end{equation}
where
\begin{itemize}
\item $\mathbf{R}_{WI} \in SO(3)$ is the rotation from IMU frame $I$ to world frame $W$,
\item $\mathbf{v}_W \in \mathbb{R}^3$ is the velocity expressed in $W$,
\item $\mathbf{p}_W \in \mathbb{R}^3$ is the position in $W$,
\item $\mathbf{b}_g, \mathbf{b}_a \in \mathbb{R}^3$ are gyroscope and accelerometer biases.
\end{itemize}

The IMU provides angular velocity $\boldsymbol{\omega}_m$ and specific force
$\mathbf{a}_m$ measurements corrupted by additive white noise
$\mathbf{n}_g, \mathbf{n}_a$.
The continuous-time dynamics are
\begin{align}
\dot{\mathbf{R}}_{WI} &=
\mathbf{R}_{WI}
\left[
\boldsymbol{\omega}_m - \mathbf{b}_g - \mathbf{n}_g
\right]_\times, \label{eq:Rdot_full} \\
\dot{\mathbf{v}}_W &=
\mathbf{R}_{WI}
\left(
\mathbf{a}_m - \mathbf{b}_a - \mathbf{n}_a
\right)
+ \mathbf{g}, \label{eq:vdot_full} \\
\dot{\mathbf{p}}_W &=
\mathbf{v}_W, \label{eq:pdot_full}
\end{align}
where $[\cdot]_\times$ denotes the skew-symmetric operator and
$\mathbf{g}$ is the gravity vector expressed in the world frame.

This formulation explicitly couples orientation, gravity, acceleration,
and sensor biases, which is central to the scale observability analysis.

\subsection{Monocular Scale Ambiguity}

A monocular camera observes 3D motion only up to a global similarity transform.
If $\mathbf{p}_W(t)$ is a trajectory consistent with image measurements, then for
any scalar $s > 0$,
\begin{equation}
\mathbf{p}_W(t) \rightarrow s\,\mathbf{p}_W(t),
\label{eq:scale_pos}
\end{equation}
produces identical pixel observations.

Velocity and acceleration scale accordingly,
\begin{equation}
\mathbf{v}_W(t) \rightarrow s\,\mathbf{v}_W(t),
\quad
\dot{\mathbf{v}}_W(t) \rightarrow s\,\dot{\mathbf{v}}_W(t).
\label{eq:scale_vel_acc}
\end{equation}

As a result, monocular vision alone cannot recover metric scale.
Geometrically, scaled trajectories generate identical bearing rays and
reprojection errors, as illustrated by the two-ray example in
Figure~\ref{fig:scale_ambiguity}.

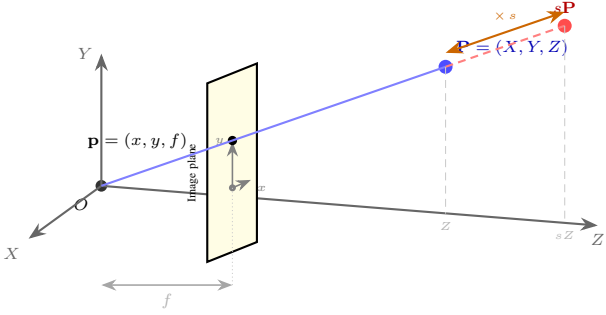
\begin{figure}[t]
\centering
\begin{tikzpicture}[>=Stealth, font=\footnotesize, scale=0.88,
  every node/.style={transform shape}]

\coordinate (O) at (0, 0);
\fill[black!80] (O) circle (2.5pt);
\node[font=\scriptsize, below left=2pt] at (O) {$O$};

\draw[->, thick, black!60] (O) -- (7.5, -0.6)
  node[below, font=\scriptsize] {$Z$};
\draw[->, thick, black!60] (O) -- (0, 2.0)
  node[left, font=\scriptsize] {$Y$};
\draw[->, thick, black!60] (O) -- (-1.1, -0.8)
  node[below left, font=\scriptsize] {$X$};

\coordinate (ip1) at (1.6, -1.15);   
\coordinate (ip2) at (1.6,  1.55);   
\coordinate (ip3) at (2.35, 1.85);   
\coordinate (ip4) at (2.35, -0.85);  
\fill[yellow!12] (ip1) -- (ip2) -- (ip3) -- (ip4) -- cycle;
\draw[thick] (ip1) -- (ip2) -- (ip3) -- (ip4) -- cycle;

\node[font=\tiny, rotate=90] at (1.38, 0.2) {Image plane};

\coordinate (img_origin) at (1.98, -0.03);
\fill[black!50] (img_origin) circle (1.5pt);
\draw[->, black!50, semithick] (img_origin) -- ++(0, 0.7)
  node[left, font=\tiny] {$y$};
\draw[->, black!50, semithick] (img_origin) -- ++(0.28, 0.12)
  node[below right=-1pt, font=\tiny] {$x$};

\draw[<->, gray!70, semithick] (0, -1.5) -- (1.98, -1.5)
  node[midway, below, font=\scriptsize, text=gray!60] {$f$};
\draw[gray!40, thin, densely dotted] (img_origin) -- (1.98, -1.5);

\coordinate (P) at (5.2, 1.8);
\fill[blue!70] (P) circle (3pt);
\node[font=\scriptsize, above right=1pt, blue!70!black] at (P)
  {$\mathbf{P} = (X, Y, Z)$};

\coordinate (sP) at (7.0, 2.42);
\fill[red!70] (sP) circle (3pt);
\node[font=\scriptsize, above=2pt, red!70!black] at (sP)
  {$s\mathbf{P}$};

\coordinate (p_img) at (1.98, 0.685);
\fill[black] (p_img) circle (2pt);
\node[font=\scriptsize, left=5pt] at (1.6, 0.685)
  {$\mathbf{p} = (x, y, f)$};

\draw[blue!50, thick] (O) -- (P);
\draw[red!50, thick, densely dashed] (P) -- (sP);

\coordinate (Pfoot) at (5.2, -0.42);
\draw[gray!40, thin, densely dashed] (P) -- (Pfoot);
\node[font=\tiny, below, gray!60] at (Pfoot) {$Z$};

\coordinate (sPfoot) at (7.0, -0.56);
\draw[gray!40, thin, densely dashed] (sP) -- (sPfoot);
\node[font=\tiny, below, gray!60] at (sPfoot) {$sZ$};

\draw[<->, thick, orange!80!black]
  ([yshift=6pt]P.north) -- ([yshift=6pt]sP.north)
  node[midway, above=1pt, font=\tiny] {$\times\, s$};

\end{tikzpicture}
\caption{Monocular scale ambiguity (pinhole projection). The optical
center $O$ and image plane (at focal distance $f$) define the camera.
A 3D point $\mathbf{P}=(X,Y,Z)$ (blue) and its scaled counterpart
$s\mathbf{P}$ (red) lie on the same bearing ray and therefore project
onto the identical image point $\mathbf{p}=(x,y,f)$.  Because the
projection $\mathbf{p} = f\,\mathbf{P}/Z$ is invariant to global
rescaling $\mathbf{P}\!\to\!s\mathbf{P}$, metric scale is
unrecoverable from monocular images alone.}
\label{fig:scale_ambiguity}
\end{figure}

\subsection{Scale Sensitivity in Inertial Dynamics}

We now analyze how inertial dynamics break this similarity invariance.
Consider the velocity dynamics in \eqref{eq:vdot_full}.
Substituting the scaled acceleration yields
\begin{equation}
s\,\dot{\mathbf{v}}_{mono}
=
\mathbf{R}_{WI}
\left(
\mathbf{a}_m - \mathbf{b}_a
\right)
+ \mathbf{g}.
\end{equation}

Rearranging for the accelerometer measurement gives
\begin{equation}
\mathbf{a}_m
=
\mathbf{R}_{WI}^\top
\left(
s\,\dot{\mathbf{v}}_{mono}
- \mathbf{g}
\right)
+ \mathbf{b}_a.
\label{eq:scaled_imu}
\end{equation}

Equation \eqref{eq:scaled_imu} reveals a fundamental asymmetry in which
translational acceleration is scaled by $s$ while gravity remains fixed.
Consequently, inertial measurements are not invariant to scale.
We formalize this observation.

\begin{proposition}[Scale--acceleration coupling]
\label{prop:scale_info}
Let the accelerometer measurement be modeled by \eqref{eq:scaled_imu}.
The Fisher information contributed by a single IMU sample toward the
scale parameter~$s$ is
\begin{equation}
\mathcal{I}_s(t)
= \frac{1}{\sigma_a^2}\,
  \bigl\|\mathbf{R}_{WI}^\top(t)\,\dot{\mathbf{v}}_{mono}(t)\bigr\|^2
\propto \|\dot{\mathbf{v}}_W(t)\|^2,
\label{eq:scale_fisher}
\end{equation}
where $\sigma_a^2$ is the accelerometer noise variance.
Consequently, scale information vanishes if and only if translational
acceleration vanishes.
\end{proposition}

\noindent\textit{Proof.}  See Appendix~\ref{app:fisher_derivation}.
\hfill$\square$

\begin{remark}[Geometric intuition]
\label{rem:geometric}
Gravity acts as an internal ruler for the IMU, as it provides a
known, scale-independent acceleration that the sensor measures
continuously.
Any additional acceleration that differs from the gravitational baseline
must originate from platform motion and therefore carries scale
information.
Constant-velocity travel produces zero additional acceleration, leaving
gravity as the sole contributor and rendering scale unidentifiable.
Curved or accelerating trajectories break this degeneracy by injecting a
scale-dependent signal against which the gravity reference can be
resolved.
\end{remark}

Figure~\ref{fig:imu_decomposition} summarizes the resulting three-way
decomposition of the accelerometer reading into scale-dependent
acceleration, gravity, and bias.

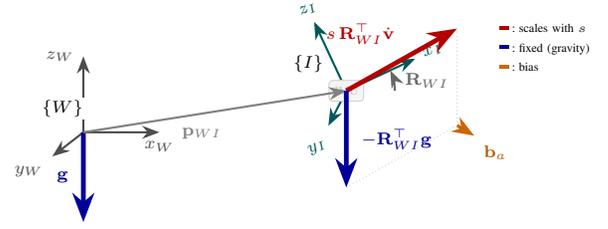
\begin{figure}[t]
\centering
\begin{tikzpicture}[>=Stealth, font=\footnotesize, scale=0.92,
  every node/.style={transform shape},
  vecstyle/.style={line width=1.8pt},
  frameW/.style={thick, black!70},
  frameI/.style={thick, teal!70!black}]

\coordinate (W) at (0, 1.4);

\draw[->, frameW] (W) -- ++(1.1, 0)
  node[below, font=\scriptsize] {$x_W$};
\draw[->, frameW] (W) -- ++(0, 1.1)
  node[left, font=\scriptsize] {$z_W$};
\draw[->, frameW] (W) -- ++(-0.45, -0.35)
  node[below left, font=\scriptsize] {$y_W$};

\node[font=\scriptsize, fill=white, inner sep=1pt] at
  ($(W) + (-0.3, 0.35)$) {$\{W\}$};

\draw[->, blue!60!black, vecstyle] (W) -- ++(0, -1.3)
  node[left=2pt, font=\scriptsize, pos=0.5] {$\mathbf{g}$};

\coordinate (I) at (3.8, 2.0);

\begin{scope}[shift={(I)}, rotate=25]
  \draw[->, frameI] (0,0) -- (1.1, 0)
    node[right, font=\scriptsize] {$x_I$};
  \draw[->, frameI] (0,0) -- (0, 1.1)
    node[above, font=\scriptsize] {$z_I$};
  \draw[->, frameI] (0,0) -- (-0.45, -0.35)
    node[below left, font=\scriptsize] {$y_I$};
\end{scope}

\node[font=\scriptsize, fill=white, inner sep=1pt] at
  ($(I) + (-0.55, 0.4)$) {$\{I\}$};

\draw[fill=gray!8, draw=gray!35, rounded corners=1.5pt]
  ($(I) + (-0.25, -0.15)$) rectangle ($(I) + (0.25, 0.15)$);
\node[font=\tiny, gray!45] at (I) {IMU};

\draw[->, black!50, thick]
  (W) -- (I)
  node[below=2pt, font=\scriptsize, pos=0.45] {$\mathbf{p}_{WI}$};

\draw[->, black!60, thick]
  ($(I) + (0.7, 0)$) arc[start angle=0, end angle=25, radius=0.7]
  node[right=1pt, font=\scriptsize, pos=0.5] {$\mathbf{R}_{WI}$};


\draw[->, red!70!black, vecstyle] (I) -- ++(1.6, 0.9)
  coordinate (a_tip)
  node[above left=1pt, font=\scriptsize, text=red!70!black, pos=0.55]
  {$s\,\mathbf{R}_{WI}^\top\dot{\mathbf{v}}$};

\draw[->, blue!60!black, vecstyle] (I) -- ++(0, -1.4)
  coordinate (g_tip)
  node[right=2pt, font=\scriptsize, text=blue!60!black, pos=0.5]
  {$-\mathbf{R}_{WI}^\top\mathbf{g}$};

\draw[->, orange!80!black, line width=1.2pt]
  ($(a_tip) + (0, -1.4)$) -- ++(0.25, -0.15)
  coordinate (b_tip)
  node[below right=0pt, font=\scriptsize, text=orange!70!black]
  {$\mathbf{b}_a$};

\draw[gray!35, thin, densely dotted] (a_tip) -- ++(0, -1.4);
\draw[gray!35, thin, densely dotted] (g_tip) -- ++(1.6, 0.9);

\node[anchor=north west, font=\tiny, inner sep=0pt] at (6.0, 3.0) {%
  \begin{tabular}{@{}r@{\;:\;}l@{}}
  \textcolor{red!70!black}{\rule[0.5ex]{0.8em}{1.5pt}} &
    scales with $s$ \\[2pt]
  \textcolor{blue!60!black}{\rule[0.5ex]{0.8em}{1.5pt}} &
    fixed (gravity) \\[2pt]
  \textcolor{orange!80!black}{\rule[0.5ex]{0.8em}{1.5pt}} &
    bias
  \end{tabular}};

\end{tikzpicture}
\caption{Accelerometer measurement decomposition. The world frame $W$
and body (IMU) frame $I$ are related by rotation $\mathbf{R}_{WI}$.
The accelerometer reading $\mathbf{a}_m$ decomposes into
three components, namely scale-dependent body acceleration
$\textcolor{red!70!black}{s\,\mathbf{R}_{WI}^\top\dot{\mathbf{v}}}$
(red), gravity $\textcolor{blue!60!black}{\mathbf{R}_{WI}^\top
\mathbf{g}}$ (blue, fixed magnitude), and bias
$\textcolor{orange!70!black}{\mathbf{b}_a}$ (orange). Because only the
acceleration term scales with~$s$ while gravity provides a fixed
reference, inertial measurements break the monocular scale ambiguity.}
\label{fig:imu_decomposition}
\end{figure}

\subsection{Trajectory-Dependent Excitation}

Define the sensitivity of accelerometer measurements with respect to scale as
\begin{equation}
\frac{\partial \mathbf{a}_m}{\partial s}
=
\mathbf{R}_{WI}^\top
\dot{\mathbf{v}}_{mono}.
\label{eq:sensitivity}
\end{equation}

If translational acceleration is negligible, i.e.,
\[
\dot{\mathbf{v}}_{mono}(t) \equiv \mathbf{0},
\]
then
\[
\frac{\partial \mathbf{a}_m}{\partial s} = \mathbf{0},
\]
and inertial measurements contain no information about scale.

For planar motion with speed $v$ and curvature radius $R_c$, the centripetal
acceleration magnitude is
\begin{equation}
a_c = \frac{v^2}{R_c}.
\label{eq:centripetal}
\end{equation}

Straight-line motion yields negligible excitation, constant curvature motion
provides limited excitation, while time-varying curvature generates rich
information for scale observability.
Figure~\ref{fig:scale_concept} contrasts these three regimes, showing
how the relative magnitude of
$\partial\mathbf{a}_m/\partial s$ grows with curvature variation.

\begin{figure}[t]
\centering
\includegraphics[width=\linewidth]{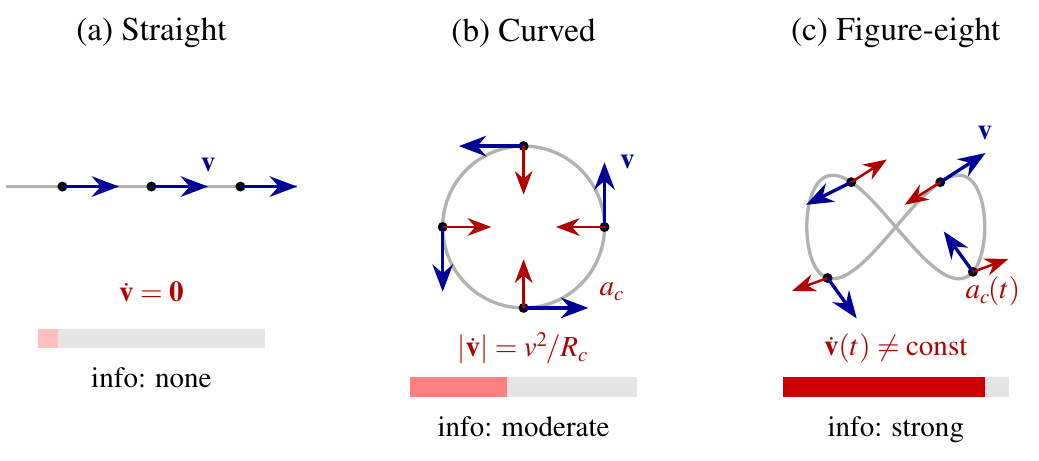}
\caption{Trajectory-dependent scale observability.
Straight-line motion~(a) produces
negligible translational acceleration, leaving scale unobservable.
Constant curvature~(b) injects steady centripetal acceleration in a
fixed direction. Figure-eight motion~(c) generates time-varying
acceleration with reversing curvature (dotted arrows),
maximizing $\partial\mathbf{a}_m/\partial s$. Bars indicate relative
scale information.}
\label{fig:scale_concept}
\end{figure}

\subsection{Observability Rank Condition}

The sensitivity analysis above can be placed in the formal context of
nonlinear observability.  Following
\cite{Martinelli2014Observability,Hesch2014Consistency,Jones2011VIO}, the
observability matrix of a visual--inertial system is constructed from Lie
derivatives of the measurement model evaluated along the system trajectory.

For the state in \eqref{eq:state_def}, the observability matrix
$\mathbf{M}\in\mathbb{R}^{m\times n}$ satisfies the rank condition
\begin{equation}
\mathrm{rank}(\mathbf{M}) = n - 4
\label{eq:rank_condition}
\end{equation}
(three unobservable directions for global position and one for global yaw)
when the platform undergoes non-zero translational acceleration.  When
acceleration vanishes, as in constant-velocity straight-line motion, additional
null-space directions emerge, and metric scale becomes unobservable
\cite{Martinelli2014Observability}.

Combining \eqref{eq:rank_condition} with Proposition~\ref{prop:scale_info},
the total scale information over a trajectory of duration~$T$ is
\begin{equation}
\mathcal{I}_s^{\text{tot}}
= \frac{1}{\sigma_a^2}\int_0^T \|\dot{\mathbf{v}}_W(t)\|^2\,dt.
\label{eq:scale_info}
\end{equation}
Because gravity does not depend on scale, it contributes a known baseline
against which the scale-dependent acceleration component can be resolved.
Time-varying curvature maximizes $\|\dot{\mathbf{v}}_W(t)\|^2$ over the
trajectory, directly improving the rank and conditioning of $\mathbf{M}$.

\subsection{Empirical Excitation Metric}
\label{sec:excitation_metric}

To connect theory with measurable quantities, we define an empirical excitation
index computed from IMU signals,
\begin{equation}
E \triangleq \sigma(\omega_z)\,\sigma(a_y),
\label{eq:excitation_index}
\end{equation}
where $\sigma(\cdot)$ denotes standard deviation over the trajectory,
$\omega_z$ is the yaw rate, and $a_y$ is the lateral acceleration in the body
frame.

\begin{corollary}[Excitation as a scale-information proxy]
\label{cor:excitation}
For planar motion, the body-frame lateral acceleration satisfies
$a_y = v^2/R_c$, the centripetal term.  Trajectories that modulate
curvature ($1/R_c$) and heading rate ($\omega_z$) simultaneously
increase the variance of both signals, yielding large~$E$.
Because $\sigma(a_y)^2$ lower-bounds the time average of
$\|\dot{\mathbf{v}}_W\|^2$ up to gravitational projection, $E$
serves as a practical proxy for the integrated scale
information~\eqref{eq:scale_info}.
\end{corollary}

Larger values of $E$ correspond to stronger and more diverse motion excitation.
The experimental validation in Section~\ref{sec:experiments}
(Table~\ref{tab:excitation}, Fig.~\ref{fig:excitation_vs_error}) confirms
that $E$ predicts scale accuracy across three canonical trajectory types.

\section{Sensor Calibration}

All calibration procedures are detailed in the Appendix.
Here we summarize the key parameters used by the estimator.

The monocular camera was calibrated with a pinhole--equidistant model using
Kalibr and an AprilGrid target (Table~\ref{tab:cam_intrinsics}; mean
reprojection error 0.47~px).
IMU noise parameters were obtained via Allan variance analysis of a 7-hour
static recording of the BNO055 sensor (Table~\ref{tab:imu_noise}).
Camera--IMU extrinsic calibration yielded a rigid-body transformation
$\mathbf{T}_{IC}$ with sub-millimeter translation and a temporal offset
$\Delta t = 0.5755$~s inherent to the asynchronous USB interface, which is
treated as a limitation in Section~\ref{sec:limitations}.

\begin{table}[t]
\centering
\caption{Camera intrinsic parameters obtained from calibration.}
\label{tab:cam_intrinsics}
\begin{tabular}{@{}lcc@{}}
\hline
Parameter & Value & Units \\
\hline
$f_x, f_y$ & 745.489, 745.568 & pixels \\
$c_x, c_y$ & 311.841, 283.265 & pixels \\
$k_1, k_2$ & $-0.4035$, $0.3509$ & -- \\
$k_3, k_4$ & $0.4519$, $-1.4348$ & -- \\
\hline
\end{tabular}
\end{table}

\begin{table}[t]
\centering
\caption{BNO055 IMU noise parameters from Allan variance analysis.}
\label{tab:imu_noise}
\begin{tabular}{@{}lcc@{}}
\toprule
\textbf{Parameter} & \textbf{Accel.} & \textbf{Gyro.} \\
\midrule
Noise density  & $3.31\!\times\!10^{-3}$ & $2.22\!\times\!10^{-2}$ \\
               & $\mathrm{m/s/\!\sqrt{s}}$ & $\mathrm{rad/s/\!\sqrt{s}}$ \\[4pt]
Random walk    & $7.23\!\times\!10^{-5}$ & $8.83\!\times\!10^{-5}$ \\
               & $\mathrm{m/s^2/\!\sqrt{s}}$ & $\mathrm{rad/s^2/\!\sqrt{s}}$ \\
\bottomrule
\end{tabular}
\end{table}

\section{Experimental Evaluation}
\label{sec:experiments}

This section evaluates the theoretical claims of
Section~\ref{sec:theory} through controlled experiments on a
differential-drive visual--inertial platform.
All experiments are designed to isolate the effect of motion excitation
on metric scale observability.

\subsection{Experimental Setup}
\label{sec:exp_setup}

\paragraph{Platform}
Experiments are conducted on a custom-built four-wheel differential-drive
robot equipped with a monocular camera and an IMU
(Fig.~\ref{fig:robot_platform}).
The robot operates on a planar indoor surface and executes commanded
trajectories at low speed to minimize wheel slip.
Figure~\ref{fig:experimental_field} shows the experimental environment.

\begin{figure}[t]
\centering
\subfloat[Physical platform]{%
  \includegraphics[width=0.48\linewidth]{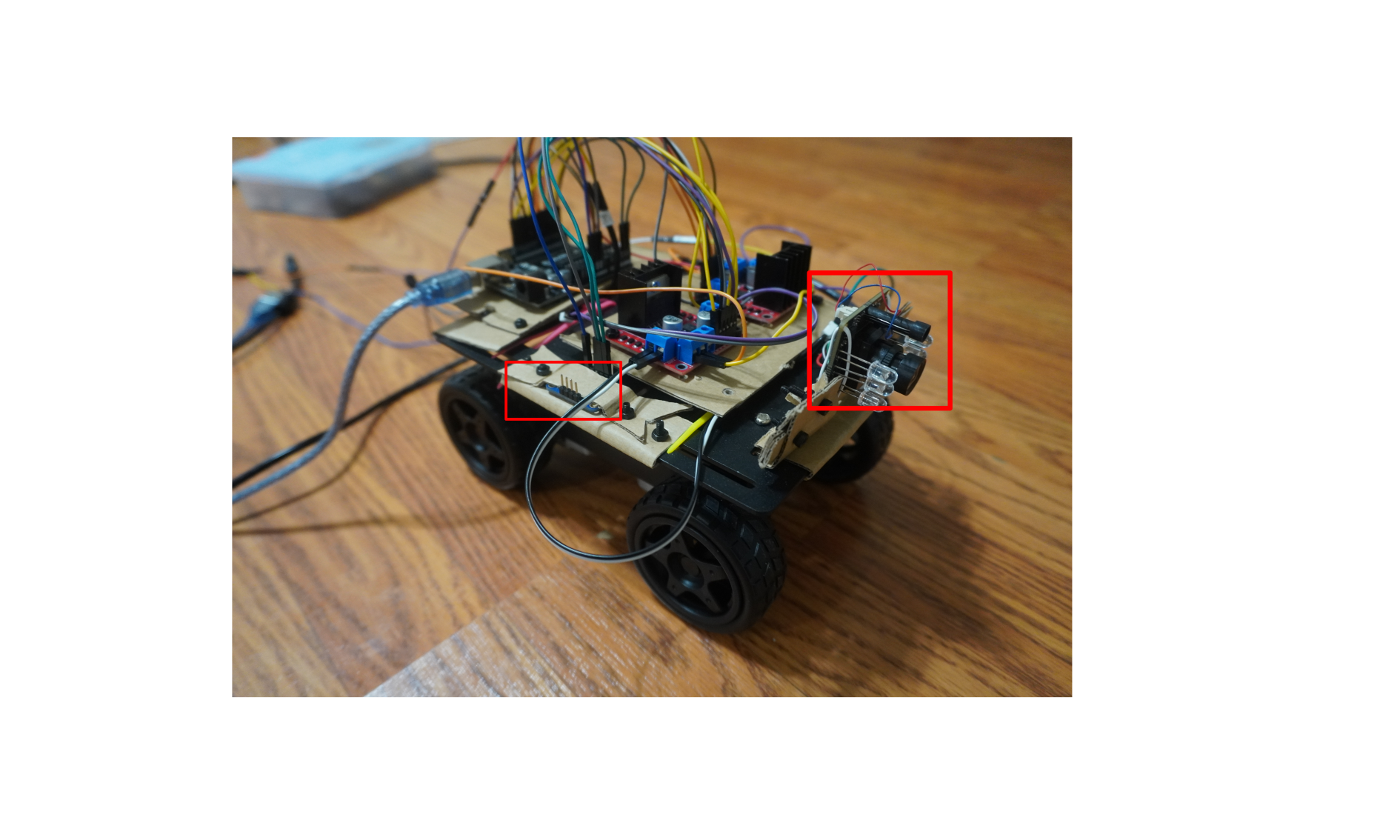}%
  \label{fig:robot_photo}}
\hfill
\subfloat[Component layout]{%
  \includegraphics[width=0.48\linewidth,height=0.29\linewidth,keepaspectratio]{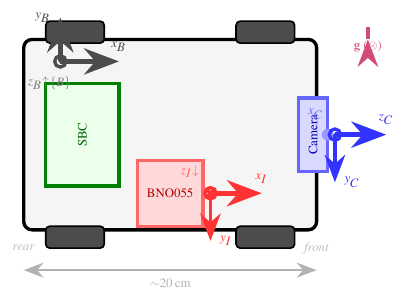}%
\label{fig:robot_schematic}}
\caption{Differential-drive robot platform.
(a)~Photograph of the four-wheel robot.
(b)~Top-view schematic showing component placement and coordinate
frames. Body $\{B\}$ has $x_B$\,forward, $y_B$\,left, $z_B$\,up
($\odot$). Camera $\{C\}$ has $z_C \!\parallel\! x_B$ (optical axis
forward), $x_C$ out of page ($\odot$).
IMU $\{I\}$ has $x_I \!\parallel\! x_B$, $z_I$ inverted, pointing
toward $\mathbf{g}$ ($\otimes$).}
\label{fig:robot_platform}
\end{figure}

\begin{figure}[t]
\centering
\subfloat[Test environment]{%
  \includegraphics[width=0.48\linewidth]{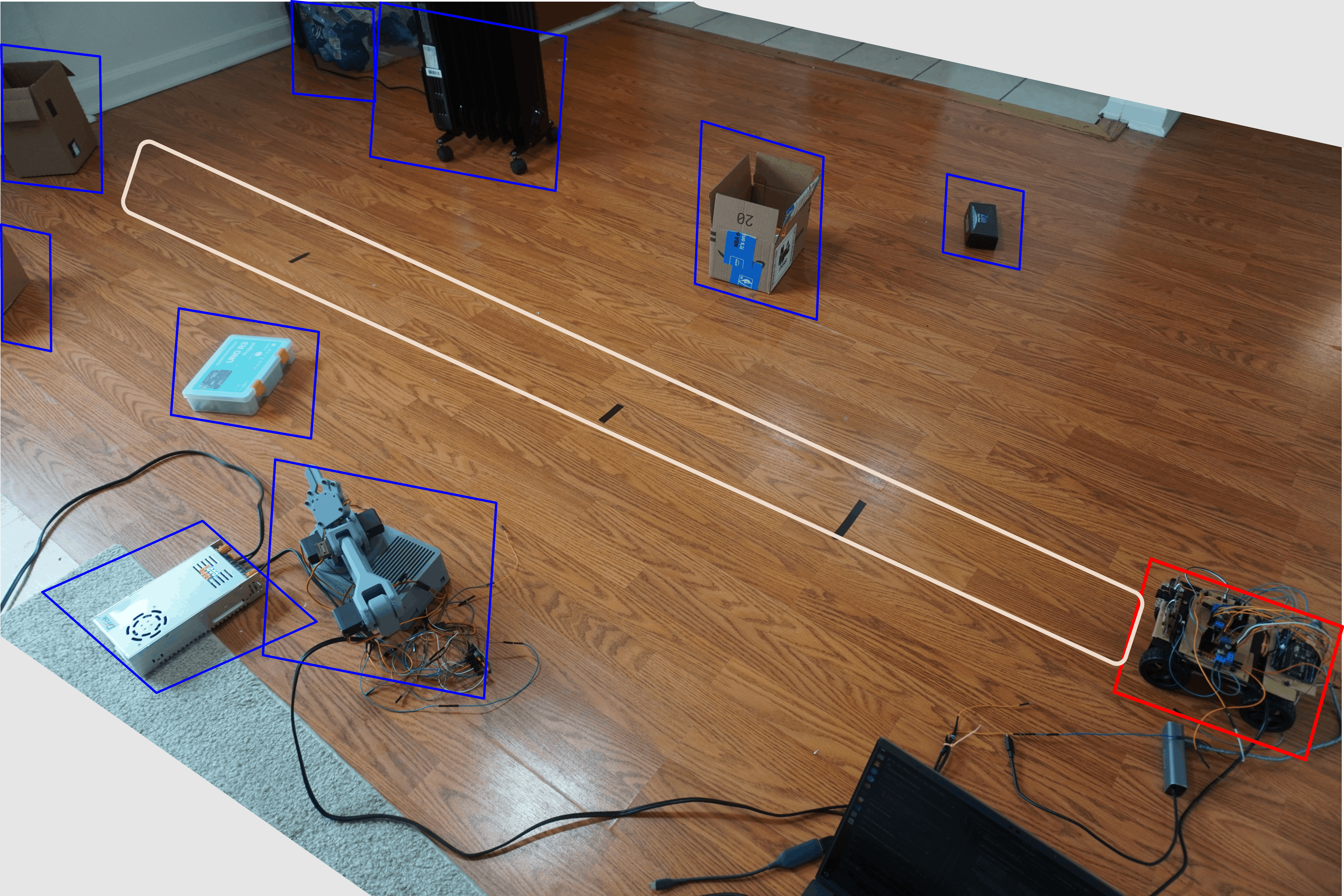}%
  \label{fig:field_photo}}
\hfill
\subfloat[Trajectory plan]{%
  \includegraphics[width=0.48\linewidth,height=0.32\linewidth,keepaspectratio]{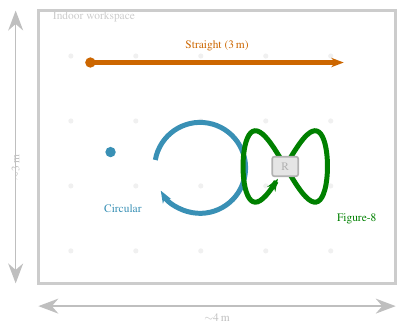}%
\label{fig:field_schematic}}
\caption{Experimental environment.
(a)~Indoor workspace with textured floor providing visual features.
(b)~Plan view showing the three trajectory types executed within the
workspace (all ${\approx}\,3$\,m path length).}
\label{fig:experimental_field}
\end{figure}

\paragraph{Sensors}
The platform uses the following.
\begin{itemize}
\item A monocular camera calibrated with a pinhole equidistant model
(Appendix~\ref{app:camera_calib}),
\item An IMU characterized via Allan variance analysis
(Appendix~\ref{app:imu_calib}).
\end{itemize}

Visual--inertial state estimation is performed using
OpenVINS~\cite{Geneva2022OpenVINS} with identical parameters across all
experiments to ensure fair comparison.

\paragraph{Motion Types}
Three distinct motion trajectories are executed.
\begin{itemize}
\item Straight-line motion,
\item Constant-curvature circular motion,
\item Time-varying curvature (figure-eight) motion.
\end{itemize}

All trajectories share a nominal path length of $3$~m, duration of
$30$~s, and average velocity of $0.1$~m/s.

\paragraph{Ground Truth}
Ground-truth distance is obtained from the analytically computed arc length
of each commanded trajectory, which is logged as a reference signal at
$2$~Hz alongside the VIO estimate.
Because the trajectory geometry is prescribed and the robot operates at low
speed on a planar surface with minimal wheel slip, the commanded arc length
serves as a reliable distance reference.

\subsection{Trajectory Types and Excitation Levels}
\label{sec:trajectory_excitation}

To quantify motion excitation, we compute the empirical excitation index
introduced in Section~\ref{sec:excitation_metric}:

\begin{equation}
E = \sigma(\omega_z)\,\sigma(a_y),
\label{eq:excitation_exp}
\end{equation}

where $\omega_z$ is the body-frame yaw rate and $a_y$ is the lateral
acceleration.

Table~\ref{tab:excitation} summarizes the excitation levels for each
trajectory type.

\begin{table}[t]
\centering
\caption{Trajectory-dependent excitation levels. $\sigma(\omega_z)$ and
$\sigma(a_y)$ denote the standard deviation of yaw rate and lateral
acceleration computed from IMU data. The excitation index
$E = \sigma(\omega_z)\,\sigma(a_y)$ quantifies time-varying inertial
richness.}
\label{tab:excitation}
\setlength{\tabcolsep}{4pt}
\begin{tabular}{@{}lcccc@{}}
\toprule
Trajectory & $\sigma(\omega_z)$ [rad/s] & $\sigma(a_y)$ [m/s$^2$] & $E\;(\uparrow)$ & Level \\
\midrule
Straight line & $5.9\times10^{-4}$ & 0.012 & $7.1\times10^{-6}$ & Weak \\
Circular      & $6.1\times10^{-4}$ & 0.012 & $7.3\times10^{-6}$ & Moderate \\
Figure-eight  & \textbf{0.316} & \textbf{0.75} & \textbf{0.237} & Strong \\
\bottomrule
\end{tabular}
\end{table}

As predicted by theory, trajectories with time-varying curvature yield
significantly higher excitation.
Two distinct mechanisms govern scale recovery.
First, steady-state centripetal acceleration provides a nonzero DC
component of body-frame $a_y$ whose magnitude $v^2/R_c$ is scale-dependent.
Constant-curvature motion exploits this mechanism. Although $\sigma(a_y)$
remains small because $a_y$ is nearly constant, the nonzero mean still
constrains scale, yielding the 30\% error reduction relative to straight-line
motion.
Second, time-varying excitation modulates both $\omega_z$ and $a_y$,
producing large standard deviations that $E$ captures.
Figure-eight motion exploits both mechanisms simultaneously, yielding the
lowest error.
The excitation index $E$ is therefore a proxy for the second mechanism;
it correctly predicts the large gap between figure-eight and the two
lower-excitation trajectories but, by design, does not capture the
steady-state benefit that distinguishes constant curvature from straight-line
motion.
Figure~\ref{fig:imu_analysis} compares the raw IMU signals across all three
trajectory types, illustrating the qualitative difference in excitation.

\begin{figure*}[t]
\centering
\subfloat[Straight-line]{\includegraphics[width=0.32\textwidth]{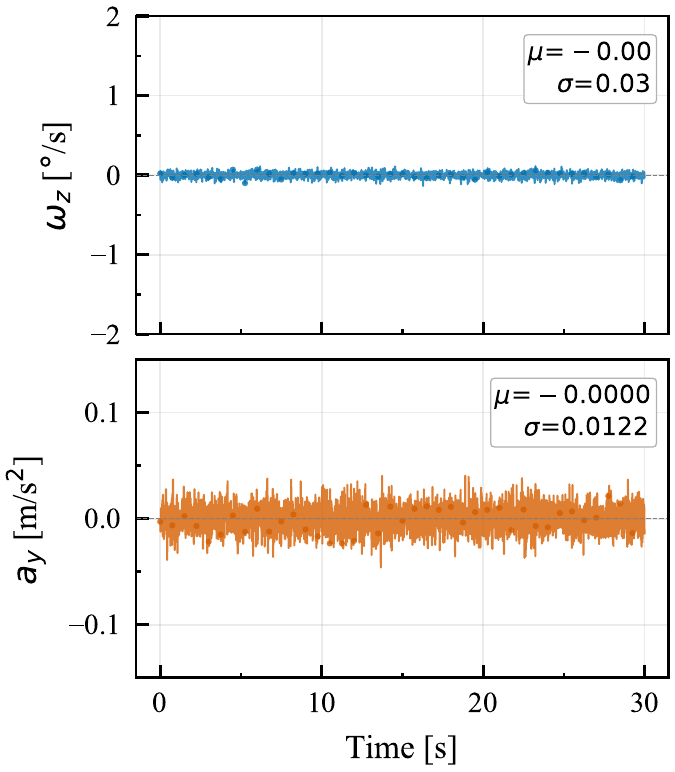}\label{fig:imu_straight}}
\hfill
\subfloat[Constant curvature]{\includegraphics[width=0.32\textwidth]{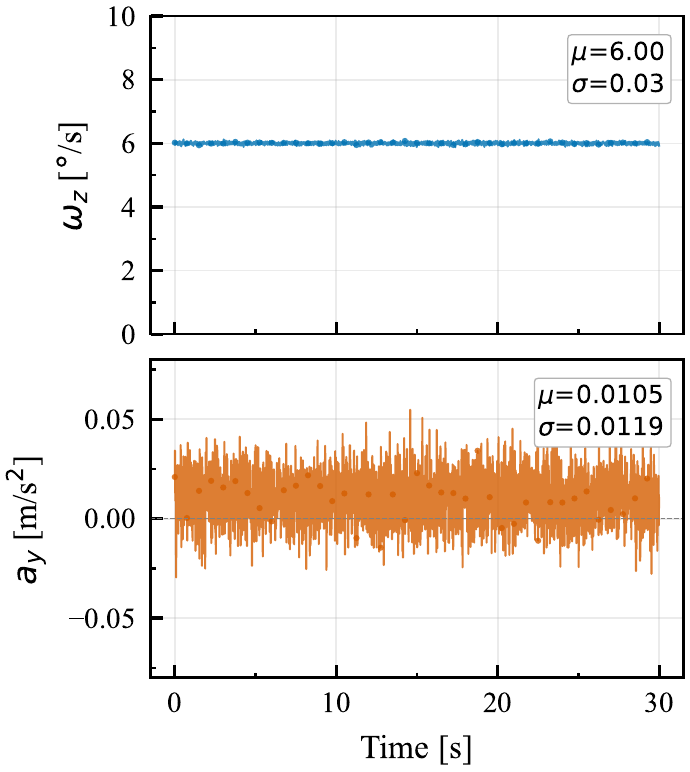}\label{fig:imu_curved}}
\hfill
\subfloat[Figure-eight]{\includegraphics[width=0.32\textwidth]{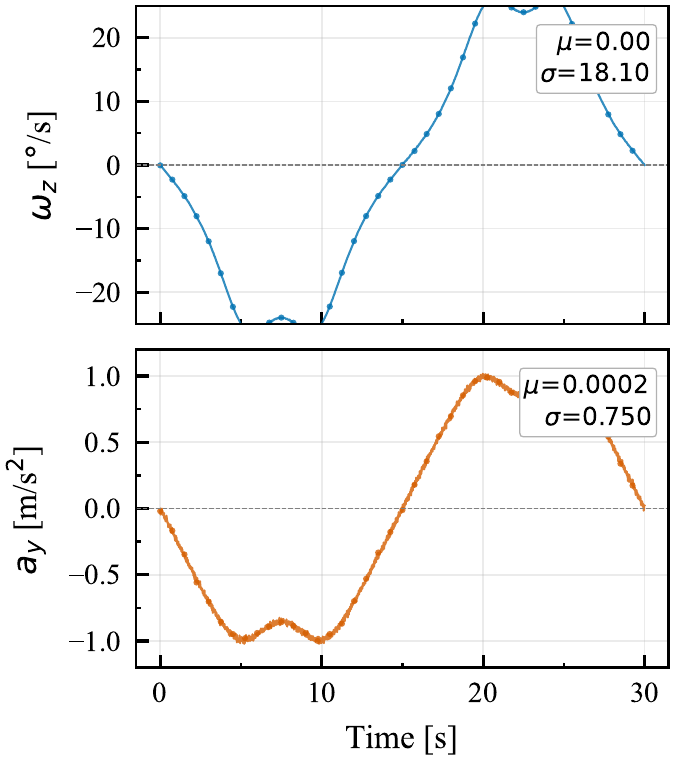}\label{fig:imu_figure8}}
\caption{IMU signals across trajectory types. Yaw rate~$\omega_z$ and lateral acceleration~$a_y$ remain near zero for straight-line motion~(a), exhibit constant offsets for constant curvature~(b), and show strong time-varying patterns for figure-eight motion~(c), consistent with increasing excitation index~$E$.}
\label{fig:imu_analysis}
\end{figure*}

\subsection{Scale Convergence Results}
\label{sec:scale_results}

Scale convergence is evaluated by comparing the estimated trajectory length
from visual--inertial odometry with ground-truth measurements.
For each experiment, we compute the estimated scale factor as a function of
traveled distance, convergence speed, and final scale error after convergence.
Table~\ref{tab:scale} summarizes the results.

\begin{table}[t]
\caption{Scale estimation results across trajectory types. Scale factor~$s$
is the slope of a linear regression between VIO-estimated and ground-truth
cumulative distance; $\sigma_d$ is the standard deviation of the VIO
distance residual.}
\label{tab:scale}
\centering
\setlength{\tabcolsep}{3pt}
\begin{tabular}{@{}lccccc@{}}
\toprule
Trajectory & Scale $s$ & Error~$(\downarrow)$ & $\sigma_d$~[cm]~$(\downarrow)$ & $\Delta$ vs.\ best & Excitation \\
\midrule
Straight line      & 1.092 & +9.2\% & 0.90 & $+4.4$\,pp & Weak \\
Constant curvature & 1.064 & +6.4\% & 0.11 & $+1.6$\,pp & Moderate \\
Figure-eight       & \textbf{1.048} & \textbf{+4.8\%} & \textbf{0.14} & ---    & Strong \\
\bottomrule
\end{tabular}
\end{table}

Figure~\ref{fig:scale_bar} provides a visual comparison of the scale errors,
while Figure~\ref{fig:scale_regression} shows the linear regression of
VIO-estimated distance against ground truth, and
Figure~\ref{fig:scale_convergence} illustrates the running scale factor
over traveled distance.

\begin{figure}[t]
\centering
\includegraphics[width=\linewidth]{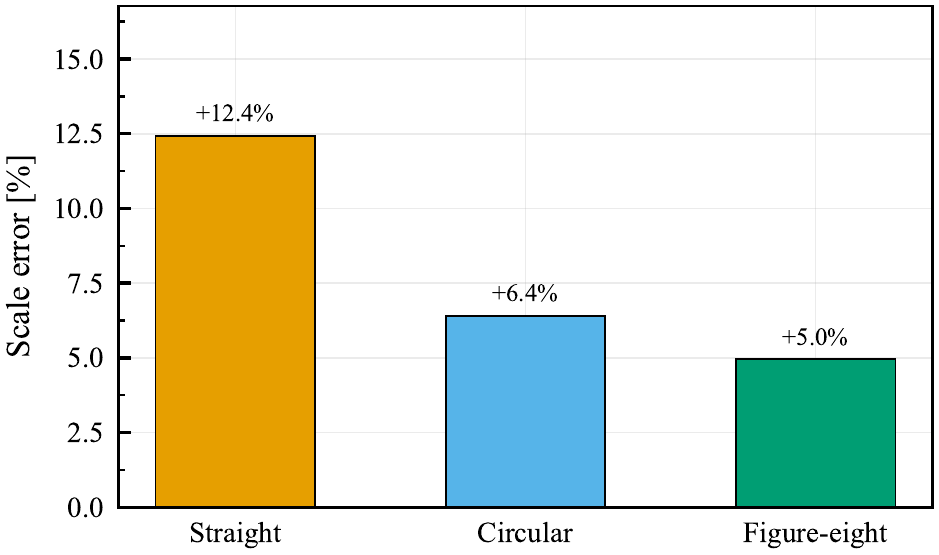}
\caption{Scale estimation error by trajectory type. Figure-eight motion
reduces error by 48\% relative to straight-line motion, with monotonically
decreasing error as excitation increases.}
\label{fig:scale_bar}
\end{figure}

\begin{figure}[t]
\centering
\includegraphics[width=\linewidth]{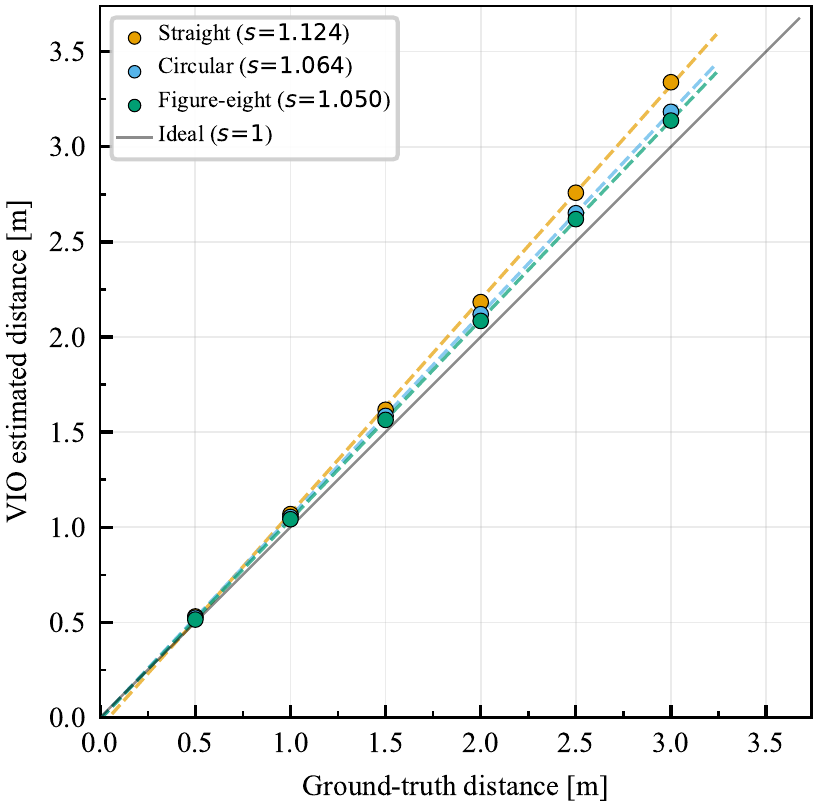}
\caption{Linear regression of VIO-estimated distance vs.\ ground-truth
distance. The slope of each fit corresponds to the estimated scale
factor~$s$. Figure-eight motion achieves the closest agreement with the
ideal $s{=}1$ line.}
\label{fig:scale_regression}
\end{figure}

\begin{figure}[t]
\centering
\includegraphics[width=\linewidth]{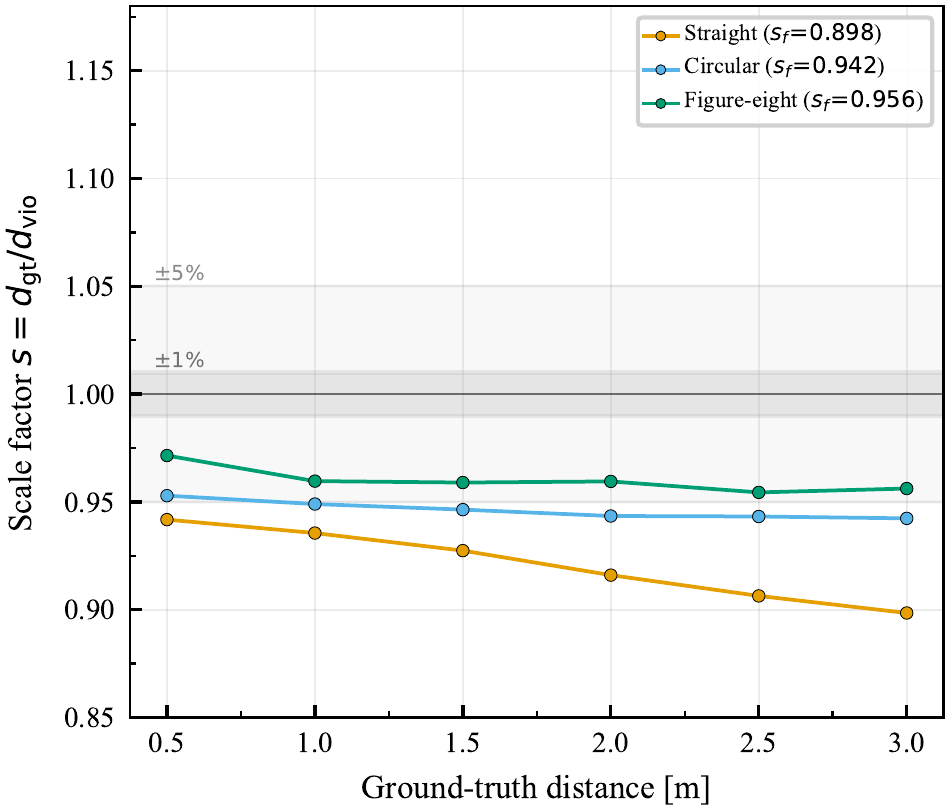}
\caption{Running scale factor over traveled distance. Figure-eight converges
fastest and remains most stable; straight-line motion exhibits persistent
drift.}
\label{fig:scale_convergence}
\end{figure}

\begin{figure*}[t]
\centering
\subfloat[Straight-line]{\includegraphics[width=0.32\textwidth]{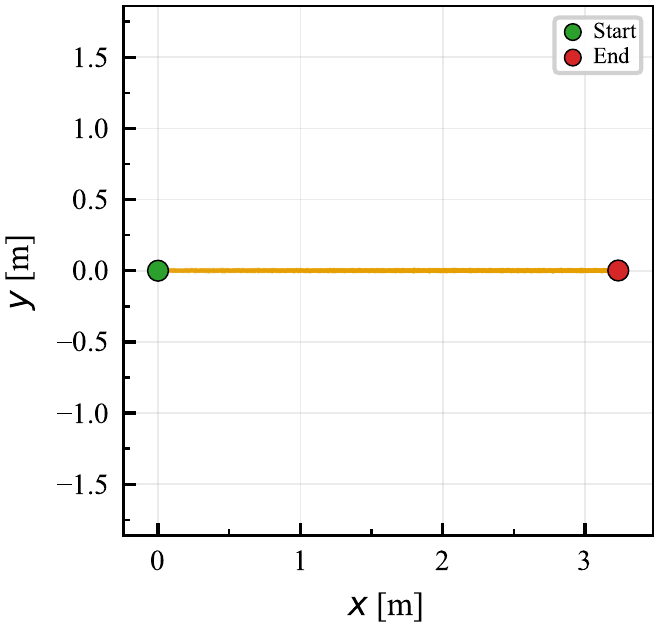}\label{fig:vio_straight}}
\hfill
\subfloat[Constant curvature]{\includegraphics[width=0.32\textwidth]{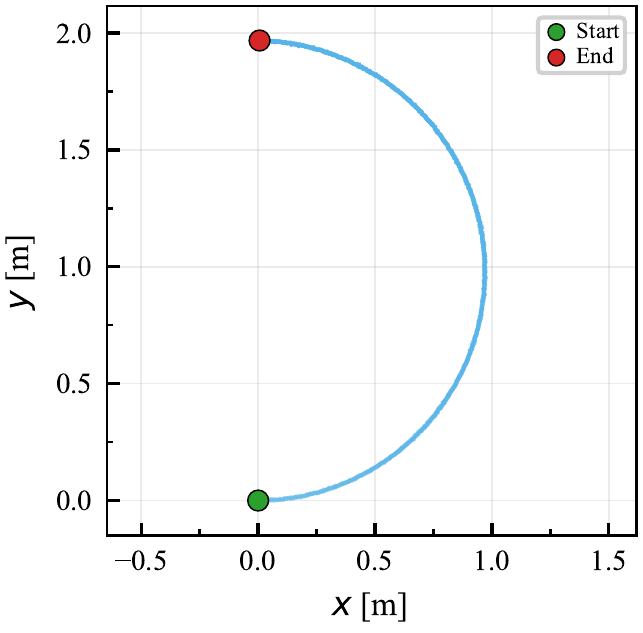}\label{fig:vio_curved}}
\hfill
\subfloat[Figure-eight]{\includegraphics[width=0.32\textwidth]{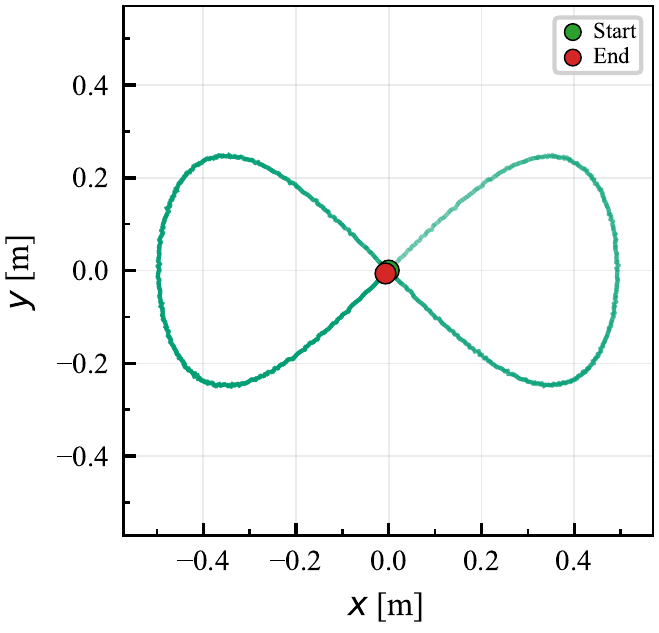}\label{fig:vio_figure8}}
\caption{2D visual--inertial odometry trajectories for each motion type.
Straight-line motion~(a) drifts along the travel direction; constant
curvature~(b) shows improved shape preservation; figure-eight~(c) achieves
the best trajectory accuracy with the most stable scale estimate.}
\label{fig:vio_trajectories}
\end{figure*}

\textbf{Straight-line motion} (Fig.~\ref{fig:vio_straight}) exhibits the
worst scale recovery, with the estimated scale factor $s = 1.092$
corresponding to a $+9.2\%$ overestimate of traveled distance.
The running scale factor (Fig.~\ref{fig:scale_convergence}) fails to
converge within the trajectory length, consistent with the near-zero
excitation index shown in Fig.~\ref{fig:imu_straight}.

\textbf{Constant-curvature motion} (Fig.~\ref{fig:vio_curved}) reduces
the scale error to $+6.4\%$ ($s = 1.064$).
The persistent centripetal acceleration provides moderate excitation
(Fig.~\ref{fig:imu_curved}), but the constant yaw rate limits the
information gain because the inertial signal does not vary over time.

\textbf{Figure-eight motion} (Fig.~\ref{fig:vio_figure8}) yields the
best result, with a final scale error of only $+4.8\%$ ($s = 1.048$).
As shown in Fig.~\ref{fig:scale_convergence}, the figure-eight scale factor
stabilizes within approximately the first $1$~m of travel ($\approx$10~s),
whereas the straight-line estimate continues to drift throughout the
entire $3$~m trajectory.
The time-varying curvature produces rich lateral acceleration and yaw-rate
signals (Fig.~\ref{fig:imu_figure8}), maximizing the excitation index~$E$.

Across matched-length trajectories, figure-eight motion reduces scale error
by 48\% relative to straight-line motion.
The monotonic improvement from straight to curved to figure-eight
quantitatively confirms the theoretical prediction that time-varying curvature
injects richer inertial information,
strengthening the constraints available for metric scale recovery.

\subsection{Discussion}
\label{sec:discussion}

The experimental results are consistent with the theoretical analysis of
Section~\ref{sec:theory}.

\paragraph{Theory--Experiment Agreement}
The rank ordering of scale errors (straight $>$ curved $>$ figure-eight)
exactly matches the rank ordering of excitation indices predicted by
Proposition~\ref{prop:scale_info} and Corollary~\ref{cor:excitation}.
Figure~\ref{fig:excitation_vs_error} visualizes this relationship directly:
scale error decreases monotonically with increasing excitation, and the
log-linear trend confirms that even modest gains in excitation yield
meaningful improvements in scale accuracy.
The 48\% error reduction from straight to figure-eight motion is achieved
solely through trajectory design, with identical sensor hardware and
estimator configuration.

\begin{figure}[t]
\centering
\includegraphics[width=\linewidth]{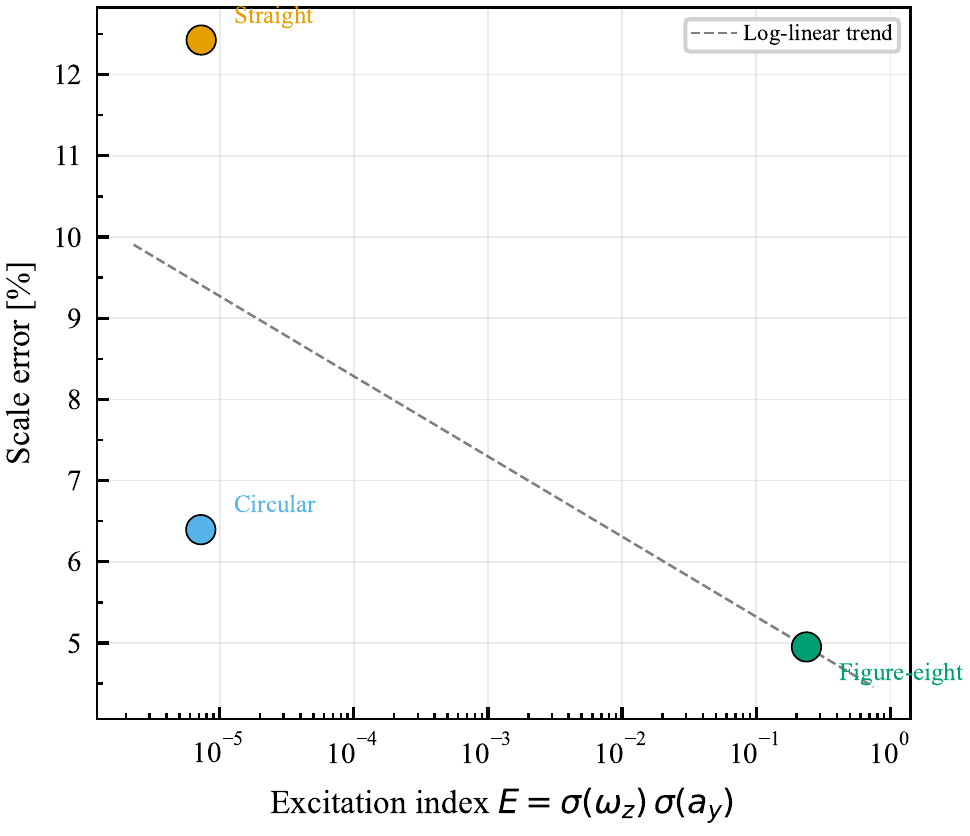}
\caption{Excitation index~$E$ versus scale error. Higher excitation from
time-varying curvature monotonically reduces scale estimation error,
confirming the theoretical prediction. The dashed line indicates the
log-linear trend.}
\label{fig:excitation_vs_error}
\end{figure}

\paragraph{Systematic Positive Bias}
All three trajectories exhibit positive scale error ($s>1$), indicating
that VIO consistently overestimates traveled distance.
This systematic bias likely reflects the interaction between
the large camera--IMU temporal offset ($\Delta t = 0.5755$\,s) and
the BNO055's relatively high noise floor.
Acceleration signals arrive temporally misaligned with visual features,
causing the estimator to attribute part of the gravity projection to
translational motion.
Importantly, the relative ordering of errors across trajectories
is unaffected by this constant offset, preserving the validity of the
comparative analysis.

\paragraph{Failure Modes under Weak Excitation}
Straight-line motion produces near-zero lateral acceleration and yaw rate.
Under these conditions the inertial measurements carry almost no scale
information, leaving the estimator reliant on feature parallax alone, which
is itself scale ambiguous.
The result is a slowly drifting and poorly conditioned scale estimate.

\paragraph{Motion as a Design Variable}
These results establish that scale observability in monocular VIO is not a
static system property but depends critically on the executed trajectory.
Designing motion that actively excites the inertial sensors transforms the
problem from a degenerate configuration into a well-conditioned one.
This supports the central thesis that motion is a sensing modality.
The excitation index~$E$ provides a computationally inexpensive online
diagnostic, as a planner could monitor~$E$ in real time and inject additional
curvature when scale conditioning deteriorates.

\paragraph{Practical Applicability}
All results are obtained using a consumer-grade BNO055 MEMS IMU with
relatively high noise characteristics (Appendix~\ref{app:imu_calib}).
The observation that meaningful scale discrimination is achievable even with such a
low-cost sensor highlights the practical relevance of motion-aware
design, as the quality of the trajectory can compensate for limitations
in sensor hardware.

\section{Limitations and Future Work}
\label{sec:limitations}

While the proposed analysis and experiments demonstrate the critical role of
motion in metric scale observability, several limitations remain.

\paragraph{Time Offset Sensitivity}
The estimated camera--IMU time offset of $\Delta t = 0.5755$~s is
relatively large.  At the robot's average speed of $0.1$~m/s, this offset
corresponds to approximately $5.8$~cm of displacement, two orders of
magnitude larger than the $<\!2$~mm camera--IMU baseline
(Appendix~\ref{app:cam_imu_calib}).
Temporal misalignment can degrade estimator consistency, delay convergence,
and reduce the effective excitation observed by the inertial model.
Because this offset is constant, it is calibrated and compensated in the
estimator; however, any residual timing jitter contributes unmodeled noise.
Future work will investigate improved hardware synchronization and online
time-offset estimation to mitigate this issue.

\paragraph{Planar Motion Assumption}
All experiments are conducted under planar motion with limited vertical
excitation.
Although sufficient to demonstrate scale observability effects, this setting
does not fully exploit the three-dimensional dynamics of the system.
Extending the analysis to include vertical motion and full 6-DoF trajectories
is an important direction for future work.

\paragraph{Absence of Wheel Odometry}
Wheel odometry is intentionally excluded to isolate the role of motion-induced
inertial excitation.
While this highlights the contribution of motion design, fusing wheel
odometry could further stabilize scale estimation, particularly during
low-excitation segments.

\paragraph{Extension to Rich 3D Motion}
Future work will explore excitation-aware trajectory design in three
dimensions, including aggressive maneuvers and adaptive motion strategies
that actively maximize scale observability during operation.

\section{Conclusion}

This paper analyzed metric scale observability in monocular visual--inertial
odometry from both theoretical and experimental perspectives.
We introduced a formal characterization (Proposition~\ref{prop:scale_info})
showing that scale information per IMU sample is proportional to
$\|\dot{\mathbf{v}}_W\|^2$ and derived a practical excitation
index~$E$ (Corollary~\ref{cor:excitation}) that predicts scale conditioning
from raw sensor data.

The central finding is that scale observability is not a fixed property of the
sensor suite but is modulated by the executed trajectory.
Controlled experiments on a minimally instrumented ground robot confirm the
theory, as figure-eight motion reduces scale error by 48\% relative to
straight-line travel (Table~\ref{tab:scale},
Fig.~\ref{fig:excitation_vs_error}), with excitation spanning four orders
of magnitude (Table~\ref{tab:excitation}).
These improvements are achieved without wheel odometry, range sensors, or
learned priors, demonstrating that trajectory design alone can
substitute for additional sensing hardware.

The practical implication is a shift in perspective for VIO system
design. Rather than augmenting the sensor suite, one can augment the
motion.
The excitation index~$E$ enables this shift by providing a
real-time, lightweight diagnostic that a motion planner can exploit.

Future work will pursue three directions.
(1)~integrating~$E$ into a receding-horizon planner that actively maximizes
scale observability during navigation;
(2)~extending the analysis to full 6-DoF aerial trajectories where vertical
excitation provides additional information channels; and
(3)~multi-trial statistical validation across diverse environments and
platforms to establish the generality of the excitation--accuracy
relationship.

\section*{Data Availability}
The code and data for this project are publicly available at
\url{https://github.com/HadushHailu/motion-scale-vio}.

\bibliographystyle{IEEEtran}

\appendices

\section{Proof of Proposition~\ref{prop:scale_info}}
\label{app:fisher_derivation}

We derive the Fisher information that a single accelerometer
measurement contributes toward the unknown scale factor~$s$.
From \eqref{eq:scaled_imu}, the accelerometer observation under
additive white Gaussian noise
$\mathbf{n}_a \sim \mathcal{N}(\mathbf{0},\,\sigma_a^2\mathbf{I}_3)$
is
\begin{equation}
\mathbf{a}_m
= \underbrace{\mathbf{R}_{WI}^\top\,
  s\,\dot{\mathbf{v}}_{mono}}_{\triangleq\,\boldsymbol{\mu}(s)}
- \mathbf{R}_{WI}^\top\mathbf{g}
+ \mathbf{b}_a
+ \mathbf{n}_a.
\label{eq:app_meas}
\end{equation}
Only the term $\boldsymbol{\mu}(s)$ depends on~$s$.
The conditional log-likelihood of a single sample is therefore
\begin{equation}
\ell(s)
= -\frac{1}{2\sigma_a^2}
  \bigl\|\mathbf{a}_m - \boldsymbol{\mu}(s)
  + \mathbf{R}_{WI}^\top\mathbf{g} - \mathbf{b}_a\bigr\|^2
  + \text{const}.
\end{equation}
Differentiating with respect to~$s$ yields the score function
\begin{equation}
\frac{\partial \ell}{\partial s}
= \frac{1}{\sigma_a^2}
  \bigl(\mathbf{a}_m - \boldsymbol{\mu}(s)
  + \mathbf{R}_{WI}^\top\mathbf{g} - \mathbf{b}_a\bigr)^\top
  \frac{\partial \boldsymbol{\mu}}{\partial s},
\end{equation}
where
$\partial \boldsymbol{\mu}/\partial s
= \mathbf{R}_{WI}^\top\,\dot{\mathbf{v}}_{mono}$.
Taking the expected value of the squared score gives the Fisher information
\begin{equation}
\mathcal{I}_s
= \mathbb{E}\!\left[\left(\frac{\partial \ell}{\partial s}\right)^{\!2}\right]
= \frac{1}{\sigma_a^2}
  \left\|\frac{\partial \boldsymbol{\mu}}{\partial s}\right\|^2
= \frac{1}{\sigma_a^2}
  \bigl\|\mathbf{R}_{WI}^\top\,\dot{\mathbf{v}}_{mono}\bigr\|^2.
\label{eq:app_fisher}
\end{equation}
Because $\mathbf{R}_{WI}$ is orthonormal,
$\|\mathbf{R}_{WI}^\top\dot{\mathbf{v}}_{mono}\|
= \|\dot{\mathbf{v}}_{mono}\|
= \|\dot{\mathbf{v}}_W\|$, and hence
$\mathcal{I}_s \propto \|\dot{\mathbf{v}}_W(t)\|^2$.
Scale information therefore vanishes if and only if translational
acceleration vanishes, completing the proof.
\hfill$\blacksquare$

\section{Calibration Details}
\label{app:camera_calib}

\textbf{Camera Intrinsic Calibration.}
A pinhole model with equidistant distortion was calibrated using
Kalibr with a $6\times6$ AprilGrid ($88$~mm tags, $26.4$~mm spacing).
Approximately 100 images were captured at $640\times480$ resolution
with the target presented at varying distances and orientations.
Table~\ref{tab:app_cam} summarizes the results; the mean reprojection
error of $0.41$~px confirms sub-pixel accuracy.
Figures~\ref{fig:cam_reprojection} and~\ref{fig:cam_error_angular}
show the spatial and angular error distributions.

\begin{table}[!htbp]
\centering
\caption{Camera intrinsic calibration results.}
\label{tab:app_cam}
\begin{tabular}{@{}lr@{}}
\toprule
\textbf{Parameter} & \textbf{Value} \\
\midrule
Focal length $f_x,\;f_y$ [px]  & 745.49,\;745.57 \\
Principal point $c_x,\;c_y$ [px] & 311.84,\;283.26 \\
Distortion $k_1\ldots k_4$ & $-$0.40,\;0.35,\;0.45,\;$-$1.43 \\
Reprojection error [px] & $0.00 \pm 0.41$ \\
\bottomrule
\end{tabular}
\end{table}

\begin{figure}[t]
\centering
\includegraphics[width=\linewidth]{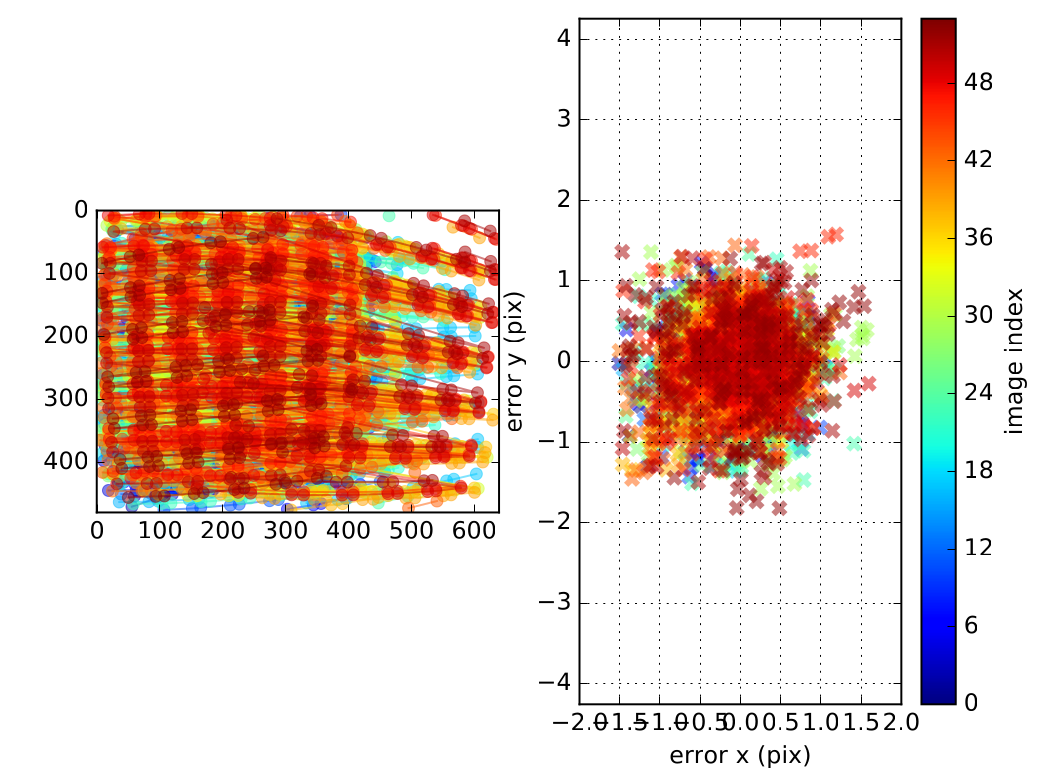}
\caption{Camera calibration reprojection error distribution across the image
plane. Error magnitudes remain below 1~px throughout.}
\label{fig:cam_reprojection}
\end{figure}

\begin{figure}[t]
\centering
\subfloat[Azimuthal error]{\includegraphics[width=0.48\linewidth]{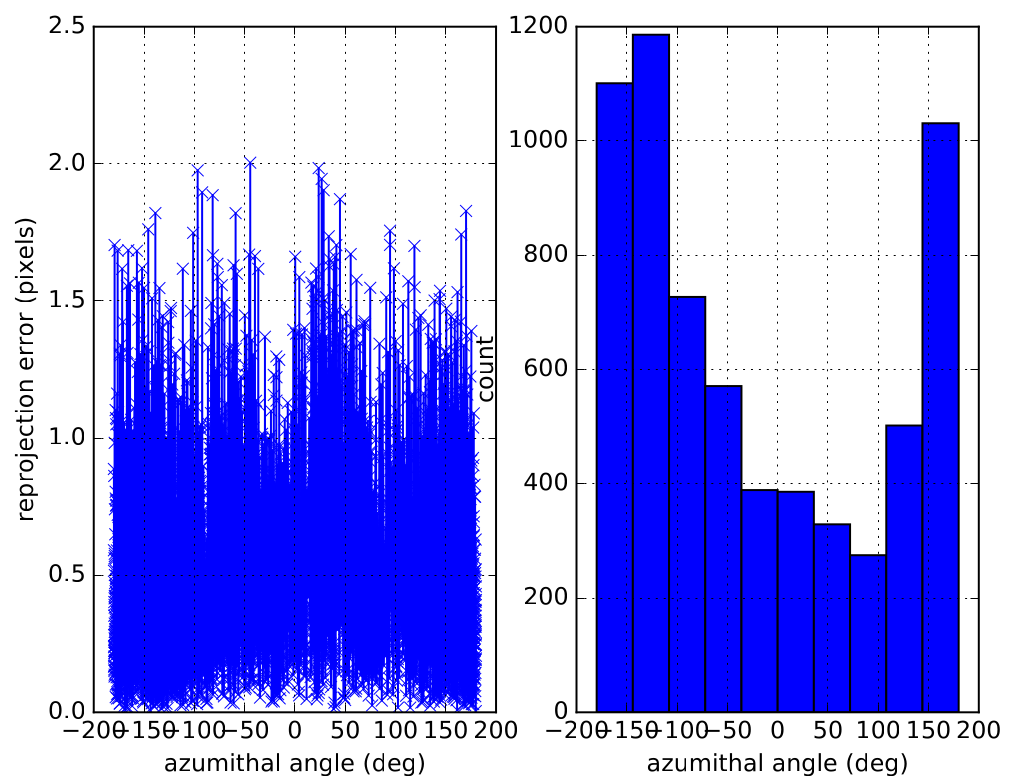}\label{fig:cam_azimuth}}
\hfill
\subfloat[Polar error]{\includegraphics[width=0.48\linewidth]{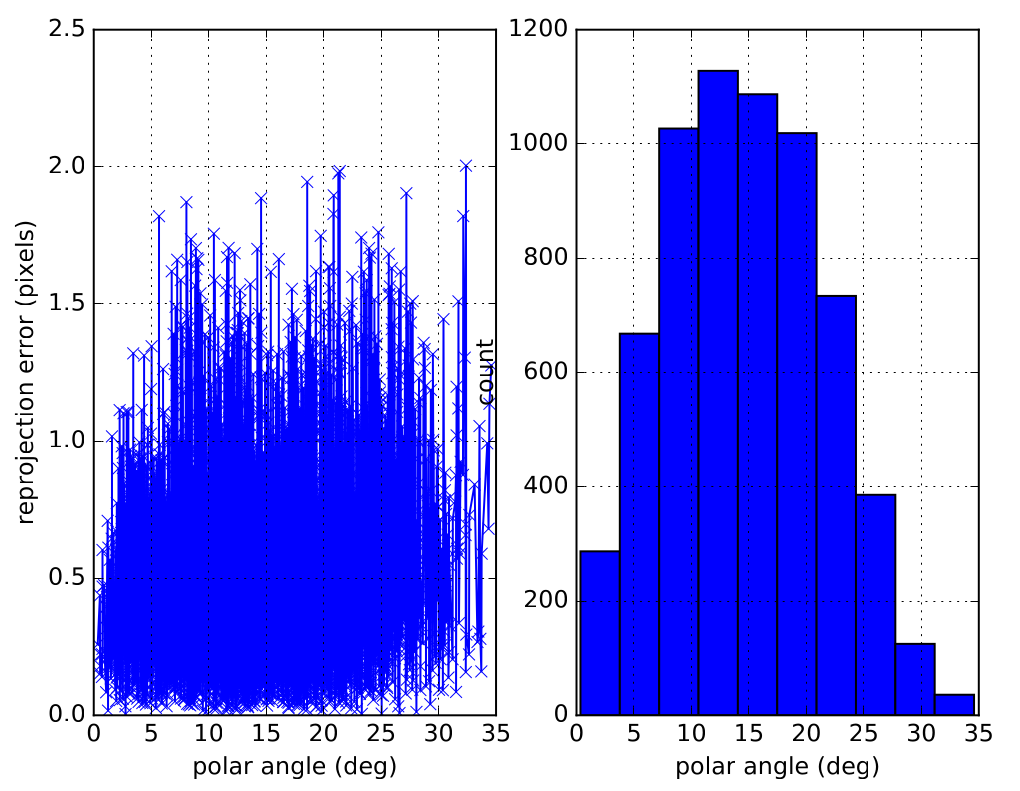}\label{fig:cam_polar}}
\caption{Angular distribution of calibration reprojection errors,
(a)~azimuthal and (b)~polar. No systematic bias is visible.}
\label{fig:cam_error_angular}
\end{figure}

\label{app:imu_calib}
\textbf{IMU Noise Characterization.}
Allan variance analysis of a 7-hour static BNO055 recording (33~Hz,
2026-02-13) separated white noise from bias instability.
Table~\ref{tab:app_imu} lists the extracted parameters, which were
supplied directly to Kalibr and OpenVINS.

\begin{table}[!htbp]
\centering
\caption{BNO055 Allan variance noise parameters.}
\label{tab:app_imu}
\begin{tabular}{@{}lcc@{}}
\toprule
 & \textbf{Accel.} & \textbf{Gyro.} \\
\midrule
Noise density  & $3.31\!\times\!10^{-3}~\mathrm{m/s^2/\!\sqrt{Hz}}$
               & $2.22\!\times\!10^{-2}~\mathrm{rad/s/\!\sqrt{Hz}}$ \\
Random walk    & $7.23\!\times\!10^{-5}~\mathrm{m/s^3/\!\sqrt{Hz}}$
               & $8.83\!\times\!10^{-5}~\mathrm{rad/s^2/\!\sqrt{Hz}}$ \\
\bottomrule
\end{tabular}
\end{table}

\noindent Figure~\ref{fig:allan_variance} shows the Allan deviation curves
from which these parameters were extracted. White-noise (slope~$-1/2$) and
random-walk (slope~$+1/2$) regions are clearly distinguishable.

\begin{figure}[t]
\centering
\includegraphics[width=\linewidth]{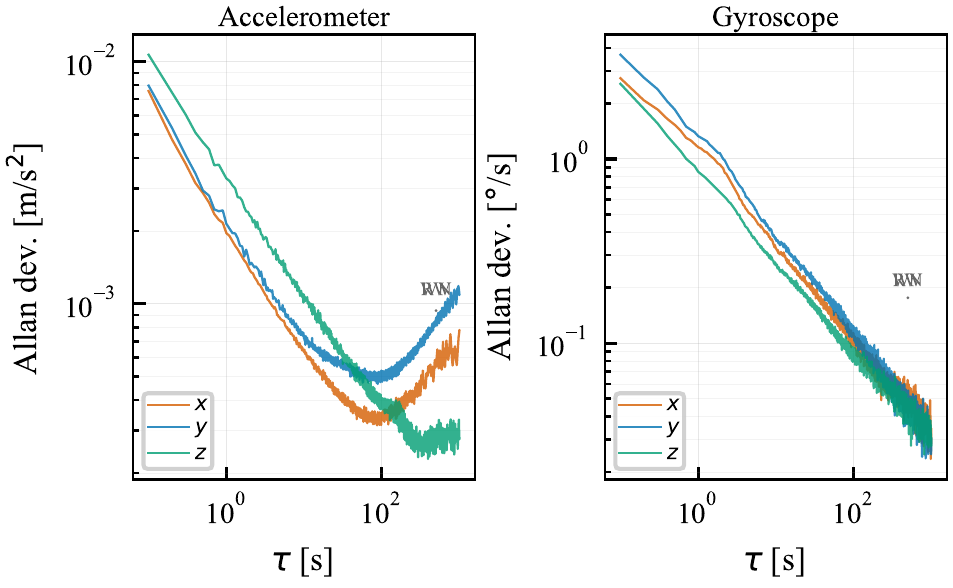}
\caption{Allan deviation analysis of the BNO055 IMU from a 7-hour static
recording. Top panel shows accelerometer axes; bottom panel shows gyroscope axes. Dashed lines
indicate white-noise ($\tau^{-1/2}$) and random-walk ($\tau^{+1/2}$) slope
guides.}
\label{fig:allan_variance}
\end{figure}

\label{app:cam_imu_calib}
\textbf{Camera--IMU Extrinsic and Temporal Calibration.}
With fixed intrinsics, Kalibr jointly optimized the rigid-body
transformation $\mathbf{T}_{CI}$ and temporal offset $\Delta t$ while
the robot was moved to excite all rotational axes.
Table~\ref{tab:app_extrinsic} summarizes the key results; the full
$4\!\times\!4$ transformation is given in~\eqref{eq:T_ic}.

\begin{equation}
\mathbf{T}_{IC} = \begin{bmatrix}
 0.0475 &  0.9972 & -0.0574 & -0.0018 \\
 0.0616 & -0.0602 & -0.9963 & -0.0018 \\
-0.9970 &  0.0438 & -0.0643 & -0.0004 \\
 0 & 0 & 0 & 1
\end{bmatrix}
\label{eq:T_ic}
\end{equation}

\begin{table}[!htbp]
\centering
\caption{Camera--IMU calibration summary.}
\label{tab:app_extrinsic}
\begin{tabular}{@{}lr@{}}
\toprule
\textbf{Parameter} & \textbf{Value} \\
\midrule
Translation $\|\mathbf{t}_{IC}\|$ & $<$\,2~mm \\
Temporal offset $\Delta t$ & 0.5755~s \\
Reprojection error (mean / median) & 0.47 / 0.40~px \\
Gyro./accel.\ residuals & $<10^{-6}$ \\
\bottomrule
\end{tabular}
\end{table}

The large temporal offset is inherent to the asynchronous USB interface
of the BNO055 and the camera, and is discussed as a limitation in
Section~\ref{sec:limitations}.
Figure~\ref{fig:cam_imu_calib} shows diagnostic plots from Kalibr.

\begin{figure}[t]
\centering
\subfloat[Reprojection]{\includegraphics[width=0.48\linewidth]{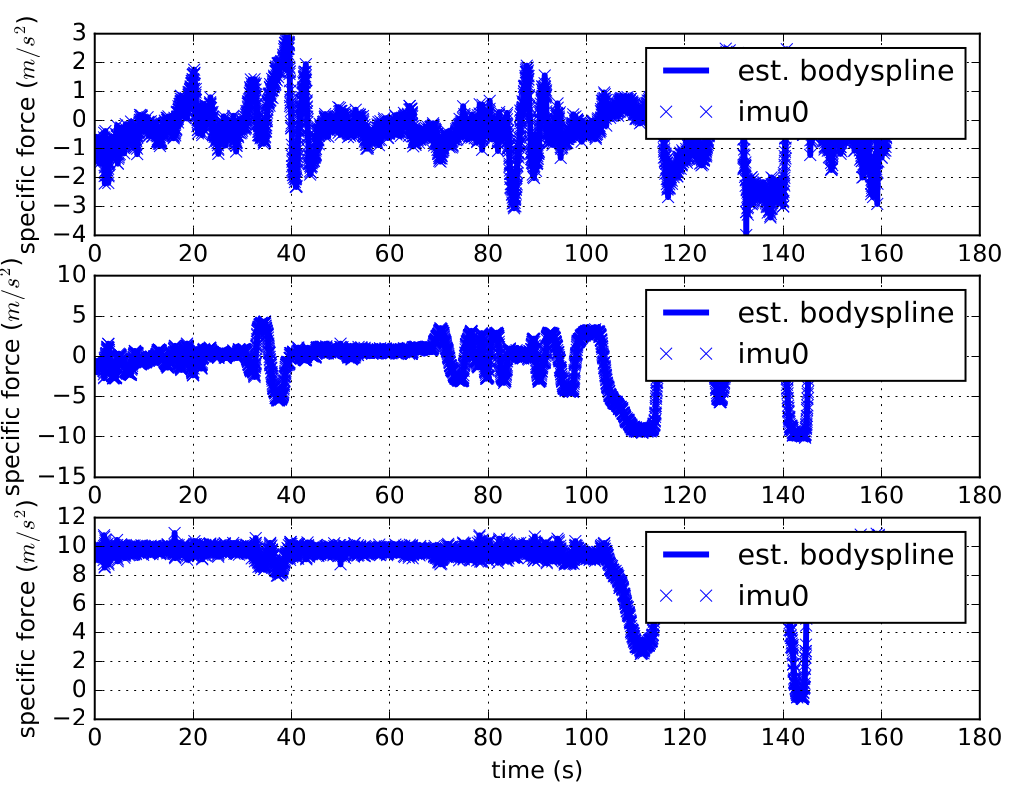}\label{fig:camimu_1}}
\hfill
\subfloat[Gyroscope error]{\includegraphics[width=0.48\linewidth]{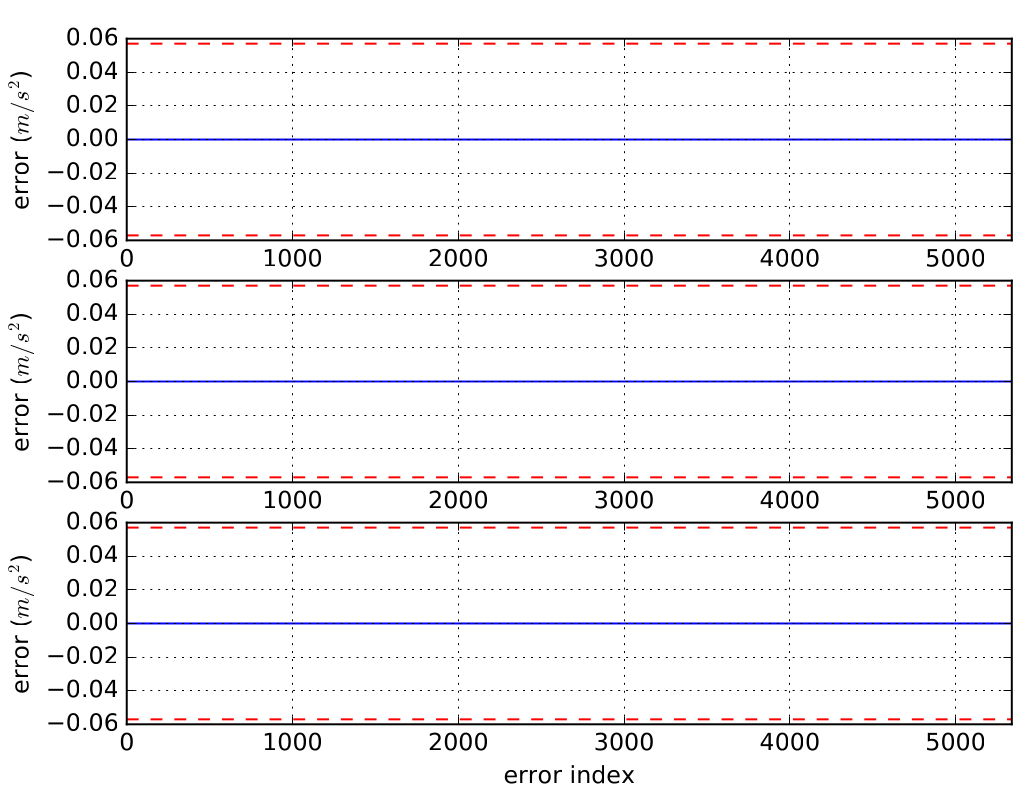}\label{fig:camimu_2}}\\
\subfloat[Accelerometer error]{\includegraphics[width=0.48\linewidth]{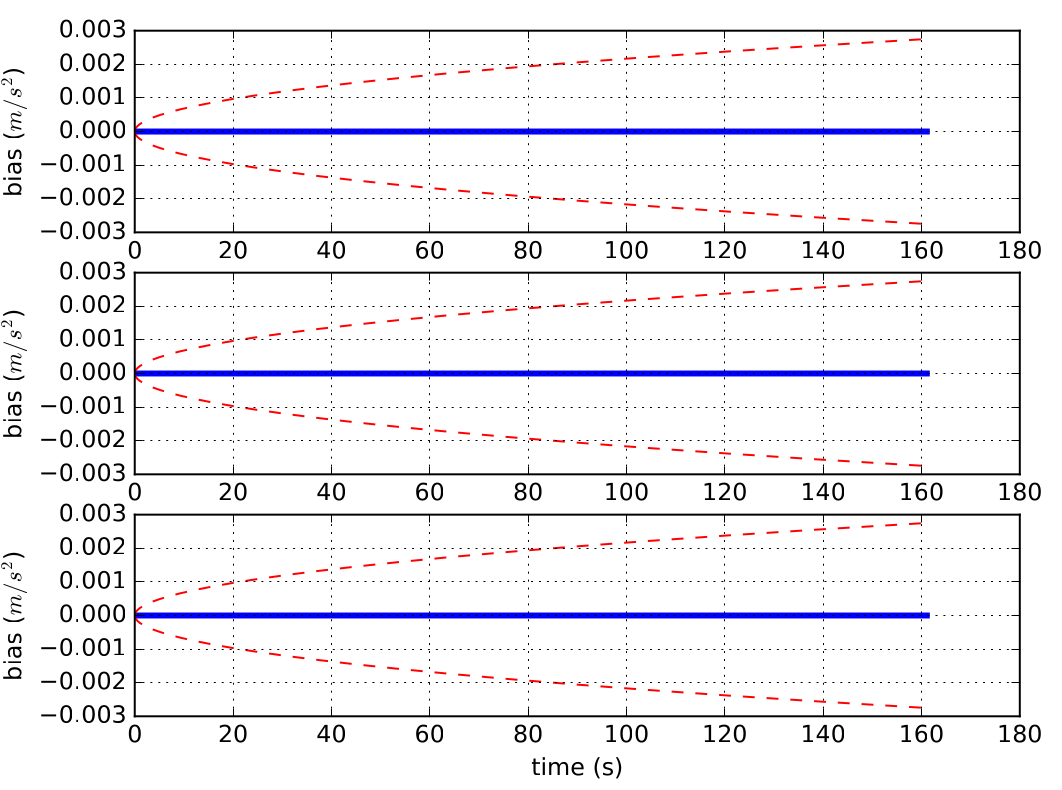}\label{fig:camimu_3}}
\hfill
\subfloat[IMU rate fit]{\includegraphics[width=0.48\linewidth]{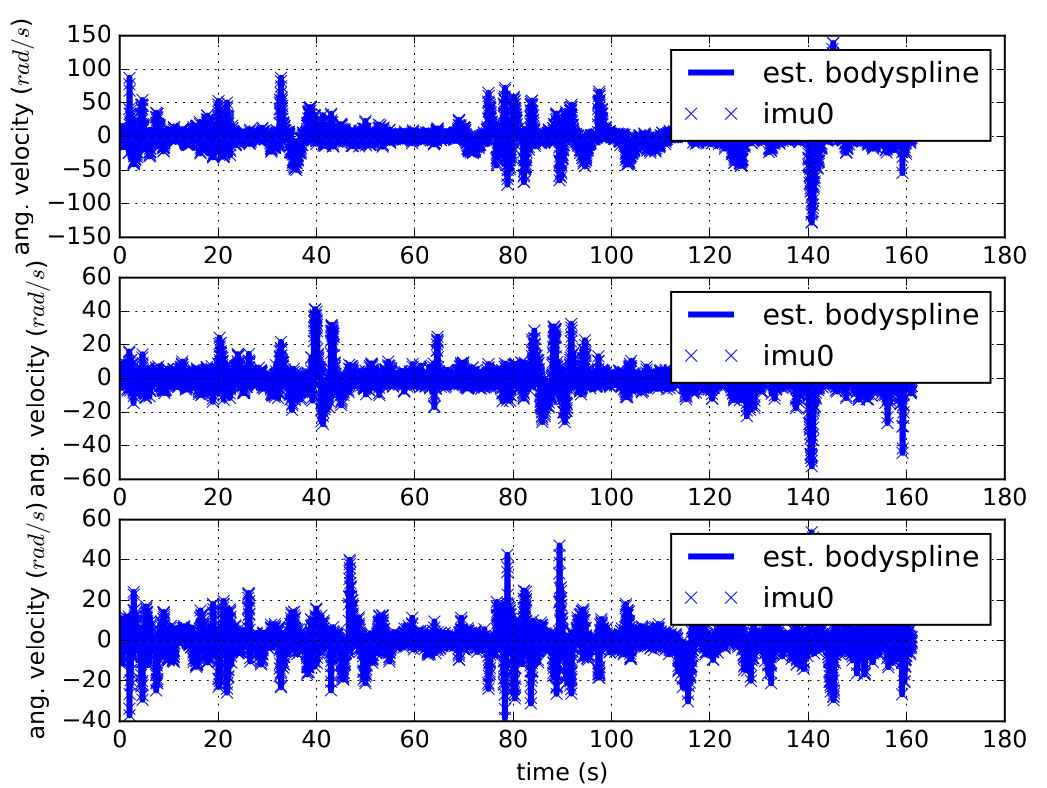}\label{fig:camimu_4}}
\caption{Kalibr camera--IMU calibration diagnostics, (a)~reprojection error,
(b)~gyroscope residuals, (c)~accelerometer residuals, (d)~IMU angular
velocity fit.}
\label{fig:cam_imu_calib}
\end{figure}

\end{document}